\documentclass[12pt,3p,times]{elsarticle}
\usepackage{multirow}
\usepackage{setspace}
\usepackage{subfigure}
\usepackage{amsthm}
\newtheorem{lemma}{Lemma}[section]
\usepackage[colorlinks, citecolor=red]{hyperref}
\usepackage{amsmath,amssymb,amsfonts}
\usepackage[linesnumbered,ruled,commentsnumbered]{algorithm2e}
\usepackage{booktabs}
\usepackage{graphicx}
\usepackage{textcomp}
\usepackage{xcolor}
\usepackage{color}
\usepackage{soul}
\usepackage{makecell}
\def\BibTeX{{\rm B\kern-.05em{\sc i\kern-.025em b}\kern-.08em
    T\kern-.1667em\lower.7ex\hbox{E}\kern-.125emX}}

\journal{Pattern Recognition}

\begin{document}
\doublespacing
\begin{frontmatter}

\title{
Kernel Correlation-Dissimilarity for Multiple Kernel k-Means Clustering
}

\author[address1]{Rina Su}\ead{rinasumath@aliyun.com}
\author[address1]{Yu Guo}\ead{yuguomath@aliyun.com}
\author[address1]{Caiying Wu}\ead{wucaiyingyun@163.com}
\author[address1]{Qiyu~Jin\corref{cor1}}
\cortext[cor1]{Corresponding author.}
\ead{qyjin2015@aliyun.com}
\author[address2]{Tieyong~Zeng}\ead{zeng@math.cuhk.edu.hk}

\address[address1]{School of Mathematical Science, Inner Mongolia University, Hohhot, China }
\address[address2]{Department of Mathematics, The Chinese University of Hong Kong, Satin, Hong Kong}

\begin{abstract}
The main objective of the Multiple Kernel k-Means (MKKM) algorithm is to extract non-linear information and achieve optimal clustering by optimizing base kernel matrices. Current methods enhance information diversity and reduce redundancy by exploiting interdependencies among multiple kernels based on correlations or dissimilarities. Nevertheless, relying solely on a single metric, such as correlation or dissimilarity, to define kernel relationships introduces bias and incomplete characterization. Consequently, this limitation hinders efficient information extraction, ultimately compromising clustering performance. To tackle this challenge, we introduce a novel method that systematically integrates both kernel correlation and dissimilarity. Our approach comprehensively captures kernel relationships, facilitating more efficient classification information extraction and improving clustering performance. By emphasizing the coherence between kernel correlation and dissimilarity, our method offers a more objective and transparent strategy for extracting non-linear information and significantly improving clustering precision, supported by theoretical rationale. We assess the performance of our algorithm on 13 challenging benchmark datasets, demonstrating its superiority over contemporary state-of-the-art MKKM techniques.

\end{abstract}

\begin{keyword}
    k-means, multiple kernel learning, consistency, Forbenius inner product, Manhattan Distance
\end{keyword}

\end{frontmatter}

\section{Introduction}

Clustering is a common practice in both machine learning and data mining, involving the categorization of data points based on their inherent similarity \cite{2023178kmeans_survey}. It is widely used in image processing \cite{yang2019deep} and pattern recognition \cite{2023164SurveyClustering}.
Over the past few decades, numerous clustering methodologies have been introduced, including k-means clustering \cite{1979}, spectral clustering \cite{2001spectral}, maximum margin clustering \cite{2004margin}, and density-based clustering \cite{1996dbscan}. Among these methods, k-means clustering \cite{1979} has emerged as a widely adopted unsupervised learning algorithm for grouping data points. This well-established approach utilizes the Euclidean distance metric to assess the similarity between samples, effectively clustering the closest points together \cite{1967kmeans}.

In recent years, several extensions to the k-means clustering algorithm have been proposed. For instance, a novel unsupervised k-means clustering algorithm was proposed, capable of automatically discovering the optimal cluster solution without the need for initialization or parameter selection \cite{2020kmeams}. Another effective clustering algorithm based on robust deep k-means was introduced. It demonstrates robustness under various layer sizes and scatter function settings, achieving outstanding clustering performance \cite{2021robustkmeans}. Additionally, a novel and highly scalable parallel clustering algorithm was presented, which combines the k-means and k-means++ \cite{2007kmeansplus} algorithms to cluster large-scale datasets and address the challenge of processing big data \cite{kmeans2023}. 
However, these algorithms exhibit limited effectiveness in clustering linearly non-separable data.

Two effective strategies offer potential solutions for these challenges. One strategy involves deep clustering, a robust technique that captures nonlinear structures within unsupervised data by merging embedding and clustering to establish an optimal nonlinear embedding space. For instance, Yang et al. \cite{2020adversarial_deepC} utilized adversarial learning methods to bolster the resilience of deep clustering. Their emphasis centered on employing adversarial training strategies to fortify deep clustering algorithms against perturbations and adversarial attacks. Diallo et al. \cite{2023PR_mvdc} introduced an auto-attention mechanism for multi-view deep embedding clustering, seeking to enhance clustering performance through the automatic assignment of varying attention weights to informative views. Additionally, Yang et al. \cite{2023DeepMCC} proposed an innovative deep multi-view clustering model integrating feature learning with clustering. This model aims to optimize representations from various viewpoints using a collaborative learning framework, enabling a more effective and comprehensive capture of complex data structures. However, challenges persist within deep clustering despite its notable achievements, including issues with interpretability and high computational complexity.

Alternatively, kernel k-means (KKM) clustering, known for its adeptness in capturing intricate structures \cite{2021289MKC}, presents an alternative effective strategy. The KKM algorithm aims to map data points to a high-dimensional feature space to achieve linear separability and subsequently utilizes the Euclidean distance metric to partition them \cite{2012KKM}. The kernel matrix is constructed from the inner products of the mapped high-dimensional data points \cite{2002kernel}, allowing the derivation of direct high-dimensional distances from the original data by combining the computation of high-dimensional mappings and inner products \cite{ZHAO2018247}. A variety of kernel functions generate diverse kernel matrices, resulting in distinct clustering outcomes. These options provide researchers with opportunities to explore innovative approaches and potential solutions \cite{2022503KKM}.

KKM-based clustering methods, despite their extensive utilization, encounter specific challenges \cite{2021multiviewsurvey}.
One primary challenge arises from the diverse nature of real-world datasets, often comprising multiple features, sources, or relationships that capture varied information and offer additional insights \cite{2009languages}.
For instance, documents might be represented in multiple languages, while web pages can provide both inherent views and external perspectives through anchor texts from inbound hyperlinks \cite{2004multi-view}.
To address this, multi-view clustering is frequently employed.
Another challenge involves the linear separability of original data in high-dimensional space, heavily reliant on the selection of kernel functions \cite{2021multiviewsurvey}.
Furthermore, the suitability of a chosen kernel function for a specific dataset is often uncertain \cite{2009languages}, leading to suboptimal clustering performance with inappropriate kernel selections \cite{2013multiview}.
Recent advancements in alternative kernel methods have introduced the concept of generating multiple kernels, known as Multiple Kernel Clustering (MKC) \cite{2004multikernel}.
This approach liberates itself from reliance on a single fixed kernel by generating multiple homogeneous or heterogeneous kernels, thereby extracting more comprehensive information from complementary datasets and enhancing clustering performance \cite{2007multikernel}.

Over the last decade, numerous MKC techniques have been developed to improve clustering algorithm performance \cite{2021multiviewsurvey}. These techniques directly utilize kernel fusion, encompassing two primary types of clustering methods: multiple kernel k-means clustering (MKKM) and multiple kernel spectral clustering \cite{2014MKspectral}. The core concept of MKKM is to precompute kernel matrices for all features and integrate the process of linear combinatorial learning of these matrices into the clustering process, forming a unified framework \cite{2012MKFC}.

Multiple kernel learning (MKL) entails employing multiple kernel functions to extract information from diverse features, thus enhancing the model's performance and generalization capability. Liu et al. \cite{2022410MKC} introduced a novel MKC approach that extends kernel evaluation-based MKL methods by integrating clustering and multiple kernel tasks into a unified optimization problem, employing a kernel evaluation metric named center kernel alignment. Nevertheless, MKC encounters diverse challenges, encompassing the selection of the optimal kernel function and the reduction of redundant information among kernels. To tackle these challenges, numerous methods based on MKKM have been suggested.

Gönen et al. \cite{2014LMKKM} introduced a localized data fusion strategy for MKKM clustering, allocating distinct weights to each kernel to capture the distinctive attributes and data noise for individual samples. \cite{2015RMKKM} utilized the $l_{21}$-norm to gauge the distance between data points and cluster centers, augmenting the method's resilience to noise and outliers. Liu et al. \cite{2016MR} introduced the notion of correlation to quantify the relationship between kernels, thereby decreasing the probability of assigning substantial weights to highly correlated kernel pairs by penalizing those that are highly redundant. Yao et al. \cite{2021RK} presented a procedure for choosing distinct kernel subsets for the composite kernel, considerably amplifying kernel diversity. Liu et al. \cite{SimpleMKKM} introduced the SimpleMKKM algorithm, rephrasing the min-max formula as a parameter-free minimization of the optimal value function, and devised a streamlined gradient descent algorithm to calculate the optimal solution.

Current approaches predominantly focus on mitigating kernel redundancy by employing correlation, as quantified through the Frobenius inner product, or dissimilarity, assessed via the squared Frobenius norm distance. While these terms are theoretically expected to exhibit negative correlation, practical scenarios might deviate from this expectation, potentially even leading to the absence of correlation. This underscores the imperative of incorporating both kernel correlation and dissimilarity within clustering tasks, integrating them into a unified framework. This strategy facilitates the efficient reduction of redundant information and the augmentation of kernel diversity.

To summarize, this paper provides the following contributions:
\begin{enumerate} 
\item
We introduce an innovative multiple kernel clustering model that incorporates kernel correlation and dissimilarity information to enhance clustering accuracy. By treating kernels as data points in a high-dimensional space and utilizing the Manhattan distance and Frobenius inner product to quantify dissimilarity and correlation, we effectively reduce kernel redundancy. This approach not only boosts the efficiency and accuracy of multiple kernel clustering but also enhances its overall performance and generalization potential.

\item
To enhance the efficiency of the solution process, we decompose the proposed model into convex subproblems and establish their convex nature. Building on this model decomposition, we devise an alternating minimization algorithm to solve it. This algorithm stands out due to its simplicity, rapid convergence, and ability to yield improved clustering results.

\item
We rigorously assess the performance of our approach through extensive comparative simulation experiments against seven leading state-of-the-art multiple kernel clustering (MKC) methods on 13 challenging datasets. Furthermore, we conduct a comprehensive analysis of the obtained outcomes.
\end{enumerate} 

The paper is structured as follows. Section \ref{section2} offers a concise overview of established MKKM algorithms. In Section \ref{section3}, we introduce our proposed model and outline the alternating optimization process. Section \ref{section4} presents experimental results, where we compare our method with seven state-of-the-art MKKM algorithms using 13 real-world datasets. The paper concludes in Section \ref{section5} by summarizing our contributions and discussing potential directions for future research.

\section{Preliminaries}
\label{sec_notation}

In this section, we introduce some necessary notations and preliminaries. Throughout this paper, matrices, vectors, and scalars are denoted by bold uppercase, bold lowercase, and lowercase letters, respectively. $\| \cdot \|_2$ denotes $\ell_2$ norm. 
The notation details are summarized in Table \ref{tab_notation}.
\begin{table}[!t]
    \caption{Description of the used notations.}
    \begin{center}
    \resizebox{\textwidth}{!}{
        \begin{tabular}{l l | l l  }
            \hline
            Notations & Meaning & Notations & Meaning \\
            \hline
            $n$ & number of samples & $\mathrm{diag}(n_1,n_2,\dots,n_c)$ & diagonal matrix with diagonal elements $n_1,n_2,\dots,n_c$ \\
            $d$ & feature of samples & $w_p$ & the $p$-th element of vector $\mathbf{w}$ \\
            $m$ & number of kernels & $\pmb{\mu}_c$ & the center vector of the $c$-th cluster \\
            $k$ & number of clusters & $\mathbf{X}\in \mathbb{R}^{d\times n}$ & data matrix \\
            $\mathbf{K}_p$ & the $p$-th kernel matrix & $\mathbf{A}(i,c)$ & the $(i,c)$-th element of matrix $\mathbf{A}$ \\
            $\mathbf{Z}$ & the cluster indicator matrix & $\mathbf{A}(p,:)$ & the $p$-th row of matrix $\mathbf{A}$ \\
            $\mathbf{1}_n$ & identity matrix & $\mathrm{Tr}(\mathbf{A})$ & trace of matrix $\mathbf{A}$ \\
            $n_c$ & the size of the cluster $\mathbf{C}_c$ & $\mathbf{C}_c$ & a data set for the $c$-th cluster \\
            \hline
        \end{tabular}}
        \label{tab_notation}
    \end{center}
\end{table}

\section{Related work}
\label{section2}

Let $\mathbf{X}=\{\mathbf{x}_i\}_{i=1}^n\ (\mathbf{x}_i\in \mathbb{R}^{d})$ be a dataset consisting of samples, and the goal of clustering is to divide the samples into $k$ clusters $\mathbf{C}=\{\mathbf{C}_1,\mathbf{C}_2,\dots,\mathbf{C}_k\}$ using a cluster indicator matrix $\mathbf{Z}\in \{0,1\}^{n\times k}$.
To introduce our research, this section recalls three classical clustering methods that are closely related to our study: Kernel k-Means (KKM) clustering \cite{2012KKM}, Multiple Kernel k-Means (MKKM) clustering \cite{2012MKFC}, and a significant MKKM-based algorithm called Matrix-Induced Regularization for MKKM (MKKM-MR) \cite{2016MR}.

\subsection{Kernel k-Means (KKM)}
\label{subsec_kkm}
To address linearly non-separable data and capture complex structures, the work in \cite{2012KKM} proposed a kernel trick for k-means, which resulted in improved clustering performance. 
By defining $\pmb{\phi}(\cdot):\mathbb{R}^d\rightarrow \mathcal{H}$  as a mapping function that maps data point $\mathbf{x}$ to a reproducing kernel Hilbert space $\mathcal{H}$, the objective of KKM is to minimize the sum of squared errors. The KKM model can be mathematically formulated as follows \cite{2012KKM}:

\begin{equation}
    \min_{\mathbf{Z}\in \{0,1\}^{n\times k}}\sum_{i=1}^n\sum_{c=1}^k \mathbf{Z}(i,c)\|\pmb{\phi}(\mathbf{x}_i)-\pmb{\mu}_c\|_2^2,\ \mathrm{s.t.}\ \sum_{c=1}^k\mathbf{Z}(i,c)=1,
    \label{eq2}
\end{equation}
where  the $\mathbf{Z}(i,c)$ 
represents the $(i,c)$-th element of matrix $\mathbf{Z}$, and $\pmb{\mu}_c=\frac{1}{n_c}\sum_{i=1}^n \mathbf{Z}(i,c)\pmb{\phi}(\mathbf{x}_i)$, $n_c=\sum_{i=1}^n \mathbf{Z}(i,c)$.
By defining the kernel function $\kappa(\mathbf{x}_i, \mathbf{x}_j) = \pmb{\phi}(\mathbf{x}_i)^\top\pmb{\phi}(\mathbf{x}_j)$ for $i, j = 1, 2, \dots, n$, problem (\ref{eq2}) can be equivalently expressed in matrix form as follows:
\begin{equation}
\min_{\mathbf{Z}\in \{0,1\}^{n\times k}} \mathrm{Tr}\left(\mathbf{K}-\mathbf{K}(\mathbf{ZL}^{\frac{1}{2}})(\mathbf{ZL}^{\frac{1}{2}})^\top\right),\
\mathrm{s.t.}\ \mathbf{Z1}_k=\mathbf{1}_n,
\label{eq3}
\end{equation}
where $\mathbf{L}=\mathrm{diag}\left(\frac{1}{n_1},\frac{1}{n_2},\dots,\frac{1}{n_k}\right)$, and the kernel matrix is defined as $\mathbf{K}(i,j) = \kappa(\mathbf{x}_i,\mathbf{x}_j)$ for $i, j = 1, 2, \dots, n$.

The proof detailing the transition from eq. (\ref{eq2}) to (\ref{eq3}) is available in \cite{2015prove}. 
By introducing $\mathbf{H}=\mathbf{ZL}^{\frac{1}{2}}$ and enabling it to take continuous real values, the core objective of the KKM method is expressed as:
\begin{equation}
\begin{split}
    &\min_{\mathbf{H}\in \mathbb{R}^{n\times k}} \mathrm{Tr}(\mathbf{K}(\mathbf{I}_n-\mathbf{HH}^\top)),\\
    &\mathrm{s.t.}\ \mathbf{H}^\top \mathbf{H}=\mathbf{I}_k.
    \label{eq4}
\end{split}
\end{equation}
The optimal solution $\mathbf{H^*}$ for problem (\ref{eq4}) comprises $k$ eigenvectors of the kernel matrix $\mathbf{K}$ corresponding to the $k$ largest eigenvalues.

\subsection{Multiple Kernel k-Means (MKKM)}

By introducing the kernel, KKM effectively solves the problem of linearly non-separable data and obtains good clustering results. However, the performance of KKM depends on the pre-selected kernel matrix. In practical applications, the lack of prior knowledge makes it difficult to determine the most suitable kernel matrix. To address this challenge, Huang et al. \cite{2012MKFC} proposed a solution that involves substituting the original kernel matrix $\mathbf{K}$ with a linear combination $\mathbf{K_w}$ of multiple base kernels. This approach, known as Multiple Kernel k-Means clustering, aims to supplement the limitations of the single kernel clustering approach.
Given a collection of multiple base kernels $\mathcal K=\{\mathbf{K}_p\}_{p=1}^m$, the kernel function of multiple kernels is defined as:
\begin{equation}
    \kappa_\mathbf{w}(\mathbf{x}_i,\mathbf{x}_j)=\pmb{\phi}_{\mathbf{w}}(\mathbf{x}_i)^\top\pmb{\phi}_{\mathbf{w}}(\mathbf{x}_j)=\sum_{p=1}^{m}w_p^2\pmb{\phi}_p(\mathbf{x}_i)^\top\pmb{\phi}_p(\mathbf{x}_j)=
    \sum_{p=1}^{m}w_p^2\mathbf{K}_p(i,j), \quad i,j = 1,2,\cdots, n,
\end{equation}
where the combined kernel matrix is denoted as $\mathbf{K_w}(i,j) = \kappa_\mathbf{w}(\mathbf{x}_i,\mathbf{x}_j)$ for $i,j=1,2,\dots,n$, consists of kernel weights represented by the vector $\mathbf{w}$, where $\mathbf{w}^\top \mathbf{1}_m=1$ and $w_p\ge 0$ for $p=1, 2, \dots, m$. Additionally, $\pmb{\phi}_\mathbf{w}(\cdot)$ refers to a collection of mapping ensembles defined as $\pmb{\phi}_\mathbf{w}(\mathbf{x})=[w_1\pmb{\phi}_1(\mathbf{x})^\top,w_2\pmb{\phi}_2(\mathbf{x})^\top,\dots,w_m\pmb{\phi}_m(\mathbf{x})^\top]^\top$.

Building upon the formulation presented in problem (\ref{eq4}), the objective of MKKM can be expressed as follows:
\begin{equation}
\begin{split}
    & \min_{\mathbf{H}\in \mathbb{R}^{n\times k},\mathbf{w}\in \mathbb{R}_+^m}\mathrm{Tr}(\mathbf{K_w}(\mathbf{I}_n-\mathbf{HH}^\top))\\
    & \mathrm{s.t.}\ \mathbf{H}^\top \mathbf{H}=\mathbf{I}_k,\ \mathbf{w}^\top \mathbf{1}_m=1,\ w_p\ge 0,\ \forall p=1, 2, \dots, m.
    \label{eq7}
\end{split}
\end{equation}
The mentioned problem (\ref{eq7}) encompasses 2 variables, $\mathbf{H}$ and $\mathbf{w}$, which can be addressed using an alternating approach. Given $\mathbf{w}$, the optimization problem (\ref{eq7}) concerning $\mathbf{H}$ can be presented as:
\begin{equation}
\begin{split}
    &\min_{\mathbf{H}\in \mathbb{R}^{n\times k}}
    \mathrm{Tr}(\mathbf{K_w(I}_n-\mathbf{HH}^\top)),\\
    &\mathrm{s.t.}\ \mathbf{H}^\top \mathbf{H}=\mathbf{I}_k.
\end{split}
\end{equation}
This corresponds to the classical KKM algorithm, and its optimal solution is determined in the same manner as described in (\ref{eq4}).
When $\mathbf{H}$ is given, the optimization problem (\ref{eq7}) concerning $\mathbf{w}$ transforms to:
\begin{equation}
\begin{split}
    &\min_{\mathbf{w}\in \mathbb{R}_+^m}\mathbf{w^\top Bw},\\
    &\mathrm{s.t.}\ \mathbf{w}^\top\mathbf{1}_m=1,\ w_p\ge 0,\ \forall p=1, 2, \dots, m.
\end{split}
\end{equation}
Here, $\mathbf{B}=\mathrm{diag}\left(\mathbf{Tr}(\mathbf{K}_1(\mathbf{I}_n-\mathbf{HH}^\top)),\mathbf{Tr}(\mathbf{K}_2(\mathbf{I}_n-\mathbf{HH}^\top)),\dots,\mathbf{Tr}(\mathbf{K}_m(\mathbf{I}_n-\mathbf{HH}^\top))\right)$, and it represents a quadratic programming (QP) problem with linear constraints. MKKM extracts a more extensive range of data information compared to KKM. It effectively addresses the limitations of single kernel clustering, resulting in improved clustering performance.

\subsection{MKKM with Matrix-Induced Regularization (MKKM-MR)}
The MKKM model faces the challenge of redundant kernel information. Liu et al. \cite{2016MR} introduced kernel correlation learning and proposed the MKKM-MR algorithm to tackle the challenge of substantial redundancy among kernels through matrix-induced regularization. This approach aids in alleviating the issue of concurrently assigning substantial weights to highly correlated pairwise kernels.
\begin{equation}
\begin{split}
    & \min_{\mathbf{H}\in \mathbb{R}^{n\times k},\mathbf{w}\in \mathbb{R}_+^m}\mathrm{Tr}(\mathbf{K}_\mathbf{w}(\mathbf{I}_n-\mathbf{HH}^\top))+\frac{\lambda}{2}\mathbf{w}^\top \mathbf{Mw}\\
    & \mathrm{s.t.}\ \mathbf{H}^\top \mathbf{H}=\mathbf{I}_k,\ \mathbf{w}^\top \mathbf{1}_m=1,\ w_p\ge 0,\ \forall p=1, 2, \dots, m,
    \label{eq11}
\end{split}
\end{equation}
where the symbol $\lambda$ represents the regularization parameter, and $\mathbf{M}\in \mathbb{R}^{m\times m}$ denotes the kernel correlation matrix, which is defined as 
$$
\mathbf{M}(p,q)=\mathrm{Tr}(\mathbf{K}_p^\top \mathbf{K}_q), \ \forall p,q = 1,2,\cdots, m
$$

The problem (\ref{eq11}) encompasses two variables, $\mathbf{H}$ and $\mathbf{w}$, and can be addressed using the alternating minimization method. When $\mathbf{w}$ is provided, the problem (\ref{eq11}) concerning $\mathbf{H}$ can be solved by utilizing the eigenvectors of the kernel matrix $\mathbf{K_w}$, as depicted in (\ref{eq4}).
Conversely, if $\mathbf{H}$ is provided, the problem (\ref{eq11}) related to $\mathbf{w}$ can be formulated as a quadratic programming (QP) problem with linear constraints:
\begin{equation}
\begin{split}
    & \min_{\mathbf{w}\in \mathbb{R}_+^m}\frac{1}{2}\mathbf{w}^\top (2\mathbf{B}+\lambda \mathbf{M})\mathbf{w},\\
    & \mathrm{s.t.}\ \mathbf{w}^\top \mathbf{1}_m=1,\ w_p\ge 0,\ \forall p=1, 2, \dots, m,
\end{split}
\end{equation}
where $\mathbf{B}=\mathrm{diag}\left(\mathbf{Tr}(\mathbf{K}_1(\mathbf{I}_n-\mathbf{HH}^\top)),\mathbf{Tr}(\mathbf{K}_2(\mathbf{I}_n-\mathbf{HH}^\top)),\dots,\mathbf{Tr}(\mathbf{K}_m(\mathbf{I}_n-\mathbf{HH}^\top))\right)$. This quadratic programming (QP) problem can be effectively solved using a specialized quadratic programming solver. 
Liu et al. \cite{2016MR} successfully reduced kernel redundancy by penalizing significant correlations within the kernel matrix. This approach resulted in improved clustering accuracy.

\section{Kernel Correlation-Dissimilarity for Multiple kernel k-Means Clustering}
\label{section3}

This section delves into the relationship between kernel correlation and dissimilarity, utilizing the ProteinFold dataset\footnote{https://xinwangliu.github.io/} as an illustrative example. This dataset consists of  694 samples and 12 kernel matrices. The kernel matrices, denoted as $\mathbf{K}_p \in \mathbb{R}^{694\times 694}$ for $ p=1,2,\dots,12$, are constructed using 7 Gaussian kernels $\exp(-\Vert \mathbf{x}_i-\mathbf{x}_j \Vert^2 / 2\sigma^2)$, 4 polynomial kernels $(a+\mathbf{x}_i^\top \mathbf{x}_j)^b$ and 1 linear kernel $(\mathbf{x}_i^\top \mathbf{x}_j)/(\Vert \mathbf{x}_i \Vert _2\cdot \Vert \mathbf{x}_j \Vert _2)$ functions, where $\mathbf{x}_i$, $\mathbf{x}_j$ denotes the original data points for $i,j=1,2,\dots,694$. Detailed parameterization is shown in the dataset. Kernel correlation is defined in the Frobenius inner product, i.e. 
\begin{equation}
    \mathbf{M}(p,q)=\mathrm{Tr}(\mathbf{K}_p^\top \mathbf{K}_q), \ \forall p,q = 1,2,\cdots, m. \label{eq_Mij}
\end{equation}
Next, we use distances to portray dissimilarity. Considering that the Manhattan distance is more robust than the Euclidean distance and produces fewer errors at outliers, we use the Manhattan distance to measure the kernel dissimilarity $\mathbf{D}$, i.e.
\begin{equation}
    \mathbf{D}(p,q)=\sum_{i,j=1}^n\left |\ \mathbf{K}_p{(i,j)}-\mathbf{K}_q{(i,j)}\right |,\ \forall p,q = 1,2,\cdots, m.
    \label{eq_dij}
\end{equation}
We compared the kernel correlation and dissimilarity between the 12 kernels separately. The findings are elaborated in Table \ref{cor_dis_compare}, unveiling the subsequent insights:

\begin{table}
    \centering
    \caption{Kernel dissimilarity  \& correlation on  ProteinFold dataset. The least dissimilarity is marked in {\color{blue} blue} and the second least dissimilarity is marked in {\color{red} red}. The marker with the highest correlation is in \textbf{bold}.
    }
    \renewcommand\arraystretch{1.1}
    \resizebox{\textwidth}{!}{
    \begin{tabular}{c|c c c c c c c c c c c c}
        \hline
         Datasets & $\mathbf{K}_{1}$ & $\mathbf{K}_{2}$ & $\mathbf{K}_{3}$ & $\mathbf{K}_{4}$ & $\mathbf{K}_{5}$ & $\mathbf{K}_{6}$ & $\mathbf{K}_{7}$ & $\mathbf{K}_{8}$ & $\mathbf{K}_{9}$ & $\mathbf{K}_{10}$ & $\mathbf{K}_{11}$ & $\mathbf{K}_{12}$ \\
           \hline
         $\mathbf{K}_{1}$ & \textcolor{blue}{0.0} \& 4.5 & 6.5 \& 3.4 & 3.7 \& 2.5 & 4.1 \& 2.9 & 3.9 \& 2.6 & 3.8 \& 2.7 & \textcolor{red}{1.5} \& 3.2 & 2.0 \& 2.5 & 2.6 \& 1.9 & 25.6 \& \textbf{5.6} & 2.6 \& 1.2 & 3.0 \& 1.5 \\
         $\mathbf{K}_{2}$ & 6.5 \& 3.4 & \textcolor{blue}{0.0} \& \textbf{32.1} & 6.9 \& 5.1 & 6.8 \& 7.0 & 6.9 \& 5.3 & 6.8 \& 5.9 & 6.6 \& 3.0 & \textcolor{red}{6.5} \& 2.9 & 6.6 \& 3.2 & 24.4 \& 22.7 & 6.8 \& 1.4 & 7.0 \& 2.4 \\
         $\mathbf{K}_{3}$ & 3.7 \& 2.5 & 6.9 \& 5.1 & \textcolor{blue}{0.0} \& 11.4 & 4.2 \& 6.4 & \textcolor{red}{2.4} \& 9.5 & 4.2 \& 5.2 & 3.8 \& 2.3 & 3.8 \& 2.1 & 3.9 \& 2.4 & 24.9 \& \textbf{14.8} & 4.0 \& 1.1 & 4.3 \& 1.7 \\
         $\mathbf{K}_{4}$ & 4.1 \& 2.9 & 6.8 \& 7.0 & 4.2 \& 6.4 & \textcolor{blue}{0.0} \& 14.2 & 4.0 \& 7.1 & \textcolor{red}{2.8} \& 10.3 & 4.1 \& 2.6 & 4.1 \& 2.6 & 4.1 \& 3.1 & 24.3 \& \textbf{22.1} & 4.4 \& 1.2 & 4.7 \& 2.0 \\
         $\mathbf{K}_{5}$ & 3.9 \& 2.6 & 6.9 \& 5.3 & \textcolor{red}{2.4} \& 9.5 & 4.0 \& 7.1 & \textcolor{blue}{0.0} \& 12.2 & 4.1 \& 5.9 & 3.9 \& 2.3 & 3.9 \& 2.1 & 4.0 \& 2.4 & 24.8 \& \textbf{15.5} & 4.2 \& 1.1 & 4.5 \& 1.7 \\
         $\mathbf{K}_{6}$ & 3.8 \& 2.7 & 6.8 \& 5.9 & 4.2 \& 5.2 & \textcolor{red}{2.8} \& 10.3 & 4.1 \& 5.9 & \textcolor{blue}{0.0} \& 12.4 & 3.9 \& 2.5 & 3.9 \& 2.4 & 3.9 \& 2.8 & 24.6 \& \textbf{18.6} & 4.1 \& 1.2 & 4.4 \& 1.9 \\
         $\mathbf{K}_{7}$ & \textcolor{red}{1.5} \& 3.2 & 6.6 \& 3.0 & 3.8 \& 2.3 & 4.1 \& 2.6 & 3.9 \& 2.3 & 3.9 \& 2.5 & \textcolor{blue}{0.0} \& 4.4 & \textcolor{red}{1.5} \& 3.2 & 2.4 \& 2.3 & 25.6 \& \textbf{6.1} & 2.5 \& 1.3 & 3.0 \& 1.7\\
         $\mathbf{K}_{8}$ & 2.0 \& 2.5 & 6.5 \& 2.9 & 3.8 \& 2.1 & 4.1 \& 2.6 & 3.9 \& 2.1 & 3.9 \& 2.4 & \textcolor{red}{1.5} \& 3.2 & \textcolor{blue}{0.0} \& 4.0 & 2.1 \& 2.7 & 25.5 \& \textbf{6.8} & 2.5 \& 1.2 & 2.9 \& 1.6 \\
         $\mathbf{K}_{9}$ & 2.6 \& 1.9 & 6.6 \& 3.2 & 3.9 \& 2.4 & 4.1 \& 3.1 & 4.0 \& 2.4 & 3.9 \& 2.8 & 2.4 \& 2.3 & \textcolor{red}{2.1} \& 2.7 & \textcolor{blue}{0.0} \& 4.6 & 25.0 \& \textbf{14.0} & 2.7 \& 1.1 & 3.2 \& 1.5\\
         $\mathbf{K}_{10}$ & 25.6 \& 5.6 & 24.4 \& 22.7 & 24.9 \& 14.8 & \textcolor{red}{24.3} \& 22.1 & 24.8 \& 15.5 & 24.6 \& 18.6 & 25.6 \& 6.1 & 25.5 \& 6.8 & 25.0 \& 14.0 & \textcolor{blue}{0.0} \& \textbf{302.4} & 26.1 \& 0.2 & 26.1 \& 0.3 \\
         $\mathbf{K}_{11}$ & 2.6 \& 1.2 & 6.8 \& 1.4 & 4.0 \& 1.1 & 4.4 \& 1.2 & 4.2 \& 1.1 & 4.1 \& 1.2 & 2.5 \& 1.3 & 2.5 \& 1.2 & 2.7 \& 1.1 & 26.1 \& 0.2 & \textcolor{blue}{0.0} \& 3.7 & \textcolor{red}{1.7} \& \textbf{3.9} \\
         $\mathbf{K}_{12}$ & 3.0 \& 1.5 & 7.0 \& 2.4 & 4.3 \& 1.7 & 4.7 \& 2.0 & 4.5 \& 1.7 & 4.4 \& 1.9 & 3.0 \& 1.7 & 2.9 \& 1.6 & 3.2 \& 1.5 & 26.1 \& 0.3 & \textcolor{red}{1.7} \& 3.9 & \textcolor{blue}{0.0} \& \textbf{6.8} \\
         \hline
    \end{tabular}}
    \label{cor_dis_compare}
\end{table}

\begin{itemize}
    \item When considering $\mathbf{K}_1$ as an illustration from the table, the utilization of the Frobenius inner product to gauge kernel correlation revealed that the correlation between $\mathbf{K}_1$ and itself is not as pronounced as the correlation between $\mathbf{K}_1$ and $\mathbf{K}_{10}$. Similarly, $\mathbf{K}_3, \mathbf{K}_4, \mathbf{K}_5, \mathbf{K}_6, \mathbf{K}_7, \mathbf{K}_8, \mathbf{K}_9,$ and $\mathbf{K}_{11}$ also exhibited comparable patterns. This indicates that relying exclusively on the Frobenius inner product for measuring kernel similarity falls short of capturing their underlying relationships comprehensively. As a result, this approach inadequately deals with the redundancy of information among multiple kernels, ultimately resulting in a reduction in MKKM clustering accuracy.
    
    \item  Kernel $\mathbf{K}_{10}$ exhibits a slight variation in dissimilarity when compared to the other kernels, with values ranging between 24.3 and 26.1. In contrast, the correlation between kernel $\mathbf{K}_{10}$ and the remaining kernels demonstrates notable fluctuations, covering a range from 0.2 to 22.7. This contrast highlights that kernel correlation and dissimilarity encapsulate different informational facets, emphasizing their distinct assessment of kernel relationships.
    
\end{itemize}

Based on previous research and the observations mentioned earlier, we hypothesize that combining kernel correlation and kernel dissimilarity to assess kernel relationships could lead to improved evaluations, thereby further enhancing clustering accuracy.

\subsection{Proposed Approach}
We present an MKKM model that includes kernel dissimilarity and correlation as regularization terms, formulated as follows:
\begin{equation}
\begin{split}
    & \min_{\mathbf{Y}\in \mathbb{R}^{m\times m},\mathbf{H}\in \mathbb{R}^{n\times k}}\   \mathrm{Tr}\left(\mathbf{K_Y}\left(\mathbf{I}_n-\mathbf{HH}^\top\right)\right)+ \alpha \left(\frac{\mathbf{Y1}_m}{m}\right)^\top \mathbf{M}\left(\frac{\mathbf{Y1}_m}{m}\right)+\beta \mathrm{Tr}(\mathbf{D}^\top \mathbf{Y})\\
    & \mathrm{s.t.}\ \mathbf{H^\top H=I}_k,\ \mathbf{1}_m^\top \mathbf{Y=1}_m^\top,\ \ \mathbf{Y}(p,q)\ge 0,\ \forall p, q=1, 2, \dots, m,
    \label{eq_KCDCMKKM}
\end{split}
\end{equation}
where $\alpha \in \{0.1, 0.2, \dots, 0.9\}$ and $\beta \in \{2^{-14},$ $ 2^{-13}, \dots, 2^{-5}\}$ are the penalty parameters for kernel correlation and kernel dissimilarity, 
 the kernel weight vector and the combined kernel are defined as follows:
\begin{equation}
    \mathbf{w}=\frac{\mathbf{Y1}_m}{m},
    \label{update_w}
\end{equation}
\begin{equation}
    \mathbf{K_Y}=\sum_{p=1}^mw_p^2\cdot \mathbf{K}_p=\frac{1}{m^2}\sum_{p=1}^m(\mathbf{Y}(p,:)\mathbf{1}_m)^2\cdot \mathbf{K}_p,
    \label{updateK_Y}
\end{equation}
the $\mathbf{M}$ represents the kernel correlation, which is defined in \eqref{eq_Mij}, $\mathbf{Y}\in \mathbb{R}^{m\times m}$ signifies the representation matrix corresponding to the kernel dissimilarity matrix $\mathbf{D}$. This matrix $\mathbf{Y}$ indicates the probability of $\mathbf{K}_i$ representing $\mathbf{K}_j$, adhering to the condition $\mathbf{Y}(p,q)\in [0,1]$ and $\sum_{q=1}^m\mathbf{Y}(p,q)=1$, where $p,q=1,2,\dots,m$. The definition of $\mathbf{D}$ is provided in equation \eqref{eq_dij}.

\IncMargin{1em}
\begin{algorithm}[t]
    \renewcommand\arraystretch{0.8}
	{\fontsize{12pt}{\baselineskip}\selectfont 
\SetAlgoNoLine
\SetKwInOut{Input}{\textbf{Input}}
\SetKwInOut{Output}{\textbf{Output}}
\SetKwInOut{Initialize}{\textbf{Initialize}}

\Input{
    Multiple Kernels $\mathcal K=\{\mathbf{K}_1,\mathbf{K}_2,\dots,\mathbf{K}_m\}$\;\\
    Number of cluster $k$\;\\
    Trade-off parameter $\alpha,\beta$\;\\
    Stop threshold $\epsilon$\;\\}
 \Output{
    Coefficients of base kernels $\mathbf{w}^*=\mathbf{w}^{(t)}$\;\\
    Continuous cluster indicator matrix $\mathbf{H}^*=\mathbf{H}^{(t)}$\;\\
    Discrete cluster indicator matrix $\mathbf{C}$ by performing k-means clustering on $\mathbf{H}$\;\\}
\Initialize{
    Compute correlation matrix $\mathbf{M}$ with eq. (\ref{eq_Mij})\;\\
    Compute dissimilarity matrix $\mathbf{D}$ with eq. (\ref{eq_dij})\;\\ 
    Indicator matrix of kernel representation $\mathbf{Y}^{(0)}$\;\\ 
    Kernel coefficient vector $\mathbf{w}^{(0)}=mean(\mathbf{Y}^{(0)},2)$ \;\\ 
    Objective function values $f=0$ \;\\ 
    $t=0$\;\\}
\Repeat
        {\text{$f^{(t+1)}-f^{(t)}\leq \epsilon$}}
        {
        Update $\mathbf{K_{Y}}^{(t+1)}$ by solving eq.\ (\ref{updateK_Y})\;
        Update $\mathbf{H}^{(t+1)}$ by solving eq.\ (\ref{eq21})\;
        Update $\mathbf{Y}^{(t+1)}$ by solving eq.\ (\ref{eq22})\;
        Update $\mathbf{w}^{(t+1)}$ by solving eq.\ (\ref{update_w})\;
        $t=t+1$\;}
        }
    \caption{Kernel Correlation-Dissimilarity for MKKM Clustering}
    \label{ag1}
\end{algorithm}
\DecMargin{1em}

\subsection{Optimization Method}
\label{section3.3}
The optimization problem (\ref{eq_KCDCMKKM}) comprises two variables: $\mathbf{H}$ and $\mathbf{Y}$, which can be efficiently solved through the alternating minimization approach. Algorithm \ref{ag1} presents the procedural framework of the proposed method.

\subsubsection{Optimizing \texorpdfstring{$\mathbf{H}$}{}}
Given $\mathbf{Y}$, the optimization problem (\ref{eq_KCDCMKKM}) with respect to $\mathbf{H}$ can be expressed equivalently as
\begin{equation}
    \begin{split}
        &\min_{\mathbf{H}\in \mathbb{R}^{n\times k}}
            \mathrm{Tr}(\mathbf{K_Y(I}_n-\mathbf{HH}^\top)),\\ 
        &\mathrm{s.t.}\ \mathbf{H^\top H=I}_k.
        \label{eq21}
    \end{split}
\end{equation}
This optimization problem (\ref{eq21}) can also be expressed as
\begin{equation}
    \begin{split}
        &\min_{\mathbf{H}\in \mathbb{R}^{n\times k}} \mathrm{Tr}(-\mathbf{H}^\top \mathbf{K_Y}\mathbf{H})\iff \max_{\mathbf{H}\in \mathbb{R}^{n\times k}} \mathrm{Tr}(\mathbf{H}^\top \mathbf{K_Y}\mathbf{H})\\
        &s.t.\ \mathbf{H^\top H=I}_k.
        \label{equalto21}
    \end{split}
\end{equation}
Problem (\ref{equalto21}) corresponds to the classical KKM algorithm. Its optimal solution involves calculating the eigenvectors corresponding to the top $k$ eigenvalues of the combined kernel matrix $\mathbf{K_Y}$, which is based on the properties of the Rayleigh quotient function \cite{1999Rayleigh}.

\subsubsection{Optimizing \texorpdfstring{$\mathbf{Y}$}{}}
With $\mathbf{H}$ given, the optimization problem involving $\mathbf{Y}$ can be formulated as: 
\begin{equation}
    \begin{split}
        & \min_{\mathbf{Y}\in \mathbb{R}^{m\times m}}\frac{1}{m^2}(\mathbf{Y1}_m)^\top (\mathbf{B}+\alpha \mathbf{M})(\mathbf{Y1}_m)+\beta \mathrm{Tr}(\mathbf{D^\top Y})\\
        & \mathrm{s.t.}\ \mathbf{1}_m^\top \mathbf{Y=1}_m^\top,\ \ \mathbf{Y}(p,q)\ge 0,\ \forall p, q=1, 2, \dots, m,
        \label{eq22}
    \end{split}
\end{equation}
where $\mathbf{B} = \mathrm{diag}\left( \mathrm{Tr}(\mathbf{K}_1(\mathbf{I}_n-\mathbf{HH}^\top)), \mathrm{Tr}(\mathbf{K}_2(\mathbf{I}_n-\mathbf{HH}^\top)), \cdots, \mathrm{Tr}(\mathbf{K}_m(\mathbf{I}_n-\mathbf{HH}^\top))\right)$. The optimization problem (\ref{eq22}) falls into the category of quadratic programming problems, and its convexity is demonstrated below.

\begin{lemma} 
\label{lemma2}
Problem (\ref{eq22}) is a convex quadratic programming problem.
\end{lemma}
\begin{proof}
Let $\mathbf{w}=\frac{\mathbf{Y1}_m}{m},\ \mathbf{A}=\mathbf{B}+\alpha \mathbf{M}$, and $r(\mathbf{w})=\mathbf{w}^\top \mathbf{A}\mathbf{w}$. Given that the matrix $\mathbf{A}$ is positive, for any $\mathbf{w},\mathbf{v}\in \mathbb{R}^m$ and $e\in (0,1)$, we have 
\begin{equation*}
\begin{split}
    &r(e\mathbf{w}+(1-e)\mathbf{v})-e\cdot r(\mathbf{w})-(1-e)\cdot r(\mathbf{v})\\
   &=-e(1-e)(\mathbf{w}-\mathbf{v})^\top \mathbf{A}(\mathbf{w}-\mathbf{v})\\
   &=-e(1-e)\sum_{p,q=1}^m\mathbf{A}(p,q)\cdot (\mathbf{w}_p-\mathbf{v}_q)^2\leq 0,
\end{split}
\end{equation*}
which indicates that the function $\frac{1}{m^2}(\mathbf{Y1}_m)^\top (\mathbf{B}+\alpha \mathbf{M}) (\mathbf{Y1}_m)$ is convex.
Furthermore, for any $e \in (0,1)$ and $\mathbf{Y}_1,\mathbf{Y}_2 \in \mathbb{R}^{m\times m}$, we have
\begin{equation*}
\mathrm{Tr}\left(\mathbf{D}^\top \left(e\mathbf{Y}_1+(1-e)\mathbf{Y}_2\right)\right) = e\cdot \mathrm{Tr}(\mathbf{D}^\top \mathbf{Y}_1) + (1-e)\cdot \mathrm{Tr}(\mathbf{D}^\top \mathbf{Y}_2).
\end{equation*}
Hence, the function $\mathrm{Tr}(\mathbf{D}^\top \mathbf{Y})$ is convex.
Let $g(\mathbf{Y}) = \mathbf{1}_m^\top \mathbf{Y-1}_m^\top$. For any $h \in (0,1)$ and $\mathbf{Y}_1,\mathbf{Y}_2 \in \mathbb{R}^{m\times m}$, we have
\begin{equation*}
g\left(h\mathbf{Y}_1+(1-h)\mathbf{Y}_2\right)=h\cdot g(\mathbf{Y}_1)+(1-h)\cdot g(\mathbf{Y}_2),
\end{equation*}
meaning that the function $g(\mathbf{Y})$ is also convex.
In conclusion, Problem (\ref{eq22}) is a convex quadratic programming problem.
\end{proof}

Various standard convex solvers can be employed for solving convex quadratic programming problems. These solvers encompass tools such as the \textit{quadprog} function within the Matlab Optimization Toolbox and the CVX convex optimization toolbox\footnote{http://cvxr.com}.

\subsection{Convergence and computational Complexity}

Our approach alternates between minimizing two convex optimization subproblems. The optimization process for these subproblems is detailed in Section \ref{section3.3}. The first subproblem (\ref{eq21}) constitutes a problem in quadratically constrained quadratic convex programming, and its optimal solution involves computing the largest eigenvalue of the kernel matrix $\mathbf{K_Y}$. The second subproblem (\ref{eq22}) presents a convex quadratic programming problem with a precise optimal solution. We utilize the alternating minimization method, addressing the two subproblems iteratively in an alternating manner to ensure the attainment of local optima. In summary, the objective function monotonically decreases with each iteration of these two subproblems until convergence.

Next, we outline the computational complexities from the following aspects:
\begin{enumerate}
\item
The computational complexity of kernel matrices ${\mathbf{K}_i}$ $(i=1,\cdots,m)$ is $\mathcal{O}(n^2m)$.
\item
The computational complexity of the kernel dissimilarity matrix $\mathbf{D}$ and the kernel similarity matrix $\mathbf{M}$ is $\mathcal{O}(n^2m^2)$.
\item
In each iteration, the computational complexity of the feature decomposition of $\mathbf{H}$ is $\mathcal{O}(n^3)$, and the computational complexity of the weight matrix $\mathbf{Y}$ is $\mathcal{O}(m^6)$.
\end{enumerate}

\begin{table}[t]
    \caption{Description of the used MKC benchmark datasets.}
    \begin{center}
    \renewcommand\arraystretch{0.8}
	{\fontsize{10pt}{\baselineskip}\selectfont 
        \begin{tabular}{c l c c c }
            \hline
           Type  & Name & Samples (n) & Kernels (m) & Classes(k) \\
            \hline
           \multirow{11}{*}{\makecell[c]{Small and \\ medium-scale \\ datasets}} 
            & Flower17 & 170: 170: 1360	& 7 & 17 \\
            & CCV & 200: 200: 2000 & 3 & 20 \\
            & ProteinFold & 694 & 12 & 27  \\
            & YALE & 165 & 5 & 15 \\
            & Caltech101 & 3060 & 10 & 102 \\
            & JAFFE & 213 & 12 & 10 \\
            & ORL & 400 & 12 & 40 \\
            & heart & 270 & 13 & 2 \\
            & Coil20 & 1440 & 12 & 20 \\
            & Handwritten\_numerals & 2000 & 6 & 10 \\
            & mfeat & 2000 & 12 & 10 \\
            \hline
            \multirow{3}{*}{\makecell[c]{Large-scale \\ datasets}}
            & MNIST10000 & 10000 & 12 & 10 \\
            & Flower102 & 8189 & 4 & 102 \\
            & USPS & 9298 & 12 & 10 \\
            \hline
        \end{tabular}
    }
    \label{tab1}
    \end{center}
\end{table}

\section{Experimental results}
\label{section4}
\subsection{Experimental Settings}
\subsubsection{Dataset Description}
\label{sec_kernelfunction}

The algorithm's performance underwent evaluation across 13 datasets, specified in Table \ref{tab1}. These datasets are categorized by size into three groups: small, medium, and large-scale. Specifically, small datasets comprise fewer than 1000 samples, medium-scale datasets encompass a range of 1000 to 5000 samples, while large-scale datasets include over 5000 samples.
The datasets cover various data types, including images sourced from repositories like YALE\footnote{http://vision.ucsd.edu/content/yale-face-database}, Flower17\footnote{www.robots.ox.ac.uk/~vgg/data/flowers/17/}, and Caltech101\footnote{www.vision.caltech.edu/Image Datasets/Caltech101/} and handwritten digits datasets like MNIST10000\footnote{http://yann.lecun.com/exdb/mnist/} and USPS\footnote{https://paperswithcode.com/dataset/usps}. For datasets such as ProteinFold, Coil20, Handwritten\_numerals, mfeat, Flower102 and heart, we utilized pre-computed kernel matrices from Xinwang Liu's page\footnote{https://xinwangliu.github.io/}, with further details available in the corresponding published papers. Conversely, the JAFFE and ORL datasets can be accessed on Liang Du's code website\footnote{https://github.com/csliangdu/RMKKM}. 
We computed the kernel matrices for these datasets using 12 foundational kernel functions employing the conventional method of multiple kernel learning:
\begin{itemize}
    \item 7 Gaussian kernels: $\kappa(\mathbf{x}_i,\mathbf{x}_j)=\exp(-\Vert \mathbf{x}_i-\mathbf{x}_j \Vert^2 / 2\sigma^2)$, where $\sigma =c \times d$, and $c \in \{0.01,0.05,0.1,$ $1,10,50,100\}$, and $d$ denotes maximum distance between pairwise samples.
    \item 4 Polynomial kernels: $\kappa(\mathbf{x}_i,\mathbf{x}_j)=(a+\mathbf{x}_i^\top \mathbf{x}_j)^b$, where $a\in \{0,1\}$ and $b\in \{2,4\}$.
    \item 1 Cosine kernel: $\kappa(\mathbf{x}_i,\mathbf{x}_j)=(\mathbf{x}_i^\top \mathbf{x}_j)/(\Vert \mathbf{x}_i \Vert _2\cdot \Vert \mathbf{x}_j \Vert _2)$.
\end{itemize}
Notably, the 12 foundational kernels must undergo normalization by employing the expression $\kappa(\mathbf{x}_i,\mathbf{x}_j)=\kappa(\mathbf{x}_i,\mathbf{x}_j)/ \sqrt{\kappa(\mathbf{x}_i,\mathbf{x}_i)\cdot \kappa(\mathbf{x}_j,\mathbf{x}_j)}$ and must be scaled within the interval $[0,1]$.

These 13 datasets encompass a wide range of samples, kernels, and classes, making them well-suited for evaluating clustering performance. We presupposed that the predetermined number of clusters, denoted as $k$, aligns with the actual number of classes.

\subsubsection{Comparison Algorithms}
To comprehensively explore the efficacy of the proposed approach, we conducted a performance evaluation against 7 additional comparative algorithms. The listed algorithms include:

\begin{itemize}
    \item Multiple Kernel k-Means (MKKM) \cite{2012MKFC}: This algorithm iteratively alternates between KKM and updating kernel weights until convergence, as detailed in the related literature.
    
    \item Average Multiple Kernel k-Means (A-MKKM): This algorithm acts as a baseline model. It computes the fused kernel $\mathbf{K}=\sum_{i=1}^m \frac{1}{m}\mathbf{K}_i$ by averaging the weights of all kernels. Then, the single-kernel k-means (KKM) algorithm is applied to this fused kernel.
    
    \item Single Best Kernel k-Means (SB-KKM): This algorithm applies KKM to each kernel and identifies the kernel with the most favorable clustering performance.
\item Localized Multiple Kernel k-Means (LMKKM) \cite{2014LMKKM}: This algorithm employs the localization method for learning kernel weights.
\item Multiple Kernel k-Means with Matrix-induced Regularization (MKKM-MR) \cite{2016MR}: This approach utilizes the correlation matrix among kernels to induce a regularized objective function, effectively mitigating redundant information within the kernels.
\item Multiple Kernel k-Means Clustering by Selecting Representative Kernels (MKKM-RK) \cite{2021RK}: This algorithm chooses a predetermined subset of kernels as representative kernels and integrates a process that induces subset selection based on dissimilarity between kernels into the MKKM clustering framework.
\item Simple Multiple Kernel k-Means (SMKKM) \cite{SimpleMKKM}: This algorithm reformulates the min-max formula of \cite{2018MM} as a minimization problem for an optimal value function. Additionally, it establishes the function's differentiability and devises a simplified gradient descent algorithm to derive the optimal solution.
\end{itemize}

\subsubsection{Parameter Selection}

Parameter tuning was carried out for the MKKM-MR \cite{2016MR} and MKKM-RK \cite{2021RK} algorithms. This involved varying the parameters within the ranges $\{2^{-15},2^{-14},\dots,2^{10}\}$ and $\{2^{-14},2^{-13},\dots,2^2\}$ respectively, as outlined in their respective references. In our proposed algorithm, we fine-tuned two parameters: $\alpha$ and $\beta$. To be precise, we systematically explored the parameter space by letting $\alpha$ vary in the range $\{0.1,0.2,\cdots,0.9\}$ and $\beta$ in the range $\{2^{-14},2^{-13},\cdots,2^{-5}\}$. The selection of optimal parameters for each method was based on identifying parameter values that resulted in the highest clustering performance.

\subsubsection{Clustering Metrics}
Two types of clustering evaluation metrics exist: internal evaluation and external evaluation. In our experiments, we utilized the external evaluation method because we had access to the true labels of the data. To evaluate the clustering results' quality and perform comparisons with other methods, we employed four widely used external evaluation measures: Accuracy (ACC), Normalized Mutual Information (NMI), Purity (PUR), and Adjusted Rand Index (ARI). Additional details regarding these measures are presented below.

Let $r_i$ and $s_i$ denote the cluster label and true label of the $i$-th data point, respectively, and let $n$ represent the total number of data points. The accuracy formula is as follows:
\begin{equation}
    \mathbf{ACC}(s_i,r_i)=\frac{\sum_{i=1}^n \delta(s_i,map(r_i))}{n},
\end{equation}
\begin{equation}
    \delta (x,y)=
\begin{cases}
1, & if\ x=y,\\
0, & otherwise.
\end{cases}
\end{equation}
Here, $map(r_i)$ corresponds to a permutation mapping function that associates each element in $r_i$ with the corresponding label in $s_i$.

Mutual information is employed to quantify the degree of agreement between two statistical distributions, particularly for two clustering results. Define $\mathbf{c}_p$ as the set of data points belonging to the $p$-th class of the true label, and $\mathbf{c}'_q$ as the set of data points obtained from the clustered label for the $q$-th class. The Mutual Information (MI) is calculated using the following formula:
\begin{equation}
\mathbf{MI}(s_i,r_i)=\sum_{p=1}^k\sum_{q=1}^{k'}\frac{n_{pq}}{n}\log\left(\frac{\frac{n_{pq}}{n}}{\frac{n_p}{n}\cdot\frac{n'_q}{n}}\right).
\end{equation}
Here, $k$ and $k'$ denote the number of classes in the true label and clustered label, respectively. 
$n_p$ and $n'_q$ represent the number of points in the sets $\mathbf{c}_p$ and $\mathbf{c}'_q$, while $n_{pq}$ indicates the number of points where the $p$-th class $\mathbf{c}_p$ of the true label and the $q$-th class $\mathbf{c}'_q$ of the clustering result overlap. 
The Normalized Mutual Information (NMI) is obtained by normalizing MI to the range of $[0,1]$:
\begin{equation}
    \mathbf{NMI}(s_i,r_i)=\frac{\mathbf{MI}(s_i,r_i)}{\sqrt{\mathbf{H}_{ent}(s_i)\cdot \mathbf{H}_{ent}(r_i)}},
\end{equation}
where $\mathbf{H}_{ent}(\cdot)$ represents the entropy function. 
Purity, on the other hand, is defined as the ratio of the number of correctly clustered samples to the total number of samples:
\begin{equation}
    \mathbf{PUR}(s_i,r_i)=\sum_{p=1}^k\frac{n_p}{n}(\max_q\frac{n_{pq}}{n_p}).
\end{equation}
NMI, ACC, and Purity have values ranging from $0$ to $1$, whereas Adjusted Rand Index (ARI) has values ranging from $-1$ to $1$, with higher values indicating better clustering results. 
ARI measures the coincidence between 2 data distributions and is calculated as:

\begin{equation}
    \mathbf{ARI}=\frac{\sum_p\sum_q\binom{n_{pq}}{2}-\dfrac{\sum_p \binom{n_p}{2}\sum_q\binom{n'_q}{2}}{\binom{n}{2}}}{\dfrac{1}{2}\left[\sum_p\binom{n_p}{2}+\sum_q\binom{n'_q}{2}\right]-\dfrac{\sum_p \binom{n_p}{2}\sum_q\binom{n'_q}{2}}{\binom{n}{2}}}.
\end{equation}
To mitigate the impact of randomness inherent in k-means, we performed 50 repetitions of all experiments and reported the average and standard deviation of the results. We selected the clustering results with the highest values for the four evaluation metrics as the basis for comparison in our experiment.

\subsection{Experimental Results}
\subsubsection{Clustering Performance of All Comparison Methods}
We first validate the performance of the algorithms on small and medium scale datasets. Tables \ref{tab_ACC}-\ref{tab_ARI} display the experimental results for four clustering evaluation metrics: ACC, NMI, PUR, and ARI, showcasing the performance comparison of the mentioned algorithms across the 10 datasets. Furthermore, we contrast the average precision and average rank of our method with those of the other comparison algorithms across all datasets, as presented in the final two rows of the four tables. The results illustrate that although our method may not surpass all others on every dataset, it attains the highest average accuracy and the lowest average rank. Based on these findings, we arrive at the following conclusions:
\begin{table}[t]
\begin{center}
\caption{
The results of all clustering methods on 10 benchmark datasets, measured in terms of Accuracy (ACC), are presented as "mean ± standard deviation". The best results are highlighted in \textbf{bold}. 
The average accuracy represents the mean ACC value of each method across all datasets. 
The average rank indicates the average ranking of each comparative method across all datasets, with lower values indicating better model performance.
}
    \resizebox{\textwidth}{!}{
    \renewcommand\arraystretch{1.2}
    \begin{tabular}{l c c c c c c c c}
    \hline
    Datasets & MKKM & AMKKM & SB-KKM & LMKKM & MKKM-MR & MKKM-RK & SimpleMKKM & Ours \\
    \hline
    heart & 0.5320$\pm$0.0053 & 0.8222$\pm$0.0000 & 0.7550$\pm$0.0396 & 0.8000$\pm$0.0000 & 0.8333$\pm$0.0000 & 0.7630$\pm$0.0000 & 0.6407$\pm$0.0000 & \textbf{0.8444}$\pm$0.0000 \\
    mfeat & 0.6285$\pm$0.0300 & 0.9411$\pm$0.0399 & 0.7148$\pm$0.0609 & 0.9319$\pm$0.0211 & 0.9311$\pm$0.0430 & 0.9269$\pm$0.0006 & 0.8940$\pm$0.0015 & \textbf{0.9583}$\pm$0.0211 \\
    Handwritten$\_$numerals & 0.6498$\pm$0.0223 & 0.7827$\pm$0.0875 & 0.7285$\pm$0.0714 & 0.9508$\pm$0.0017 & 0.9536$\pm$0.0006 & 0.9388$\pm$0.0167 & 0.9437$\pm$0.0011 & \textbf{0.9606}$\pm$0.0002 \\
    Flower17 & 0.4366$\pm$0.0186 & 0.5121$\pm$0.0133 & 0.3515$\pm$0.0196 & 0.4662$\pm$0.0145 & 0.5772$\pm$0.0114 & 0.5711$\pm$0.0072 & 0.5738$\pm$0.0123 & \textbf{0.5838}$\pm$0.0149 \\
    ProteinFold & 0.2731$\pm$0.0105 & 0.2939$\pm$0.0149 & 0.3030$\pm$0.0219 & 0.2868$\pm$0.0152 & 0.3507$\pm$0.0156 & 0.3496$\pm$0.0134 & 0.3200$\pm$0.0137 & \textbf{0.3710}$\pm$0.0156 \\
    YALE & 0.5247$\pm$0.0273 & 0.5470$\pm$0.0269 & 0.4633$\pm$0.0443 & 0.5459$\pm$0.0307 & 0.5622$\pm$0.0262 & 0.5561$\pm$0.0297 & 0.5604$\pm$0.0249 & \textbf{0.5696}$\pm$0.0237 \\
    Caltech101 & 0.3620$\pm$0.0089 & 0.3705$\pm$0.0098 & 0.2673$\pm$0.0092 & 0.3539$\pm$0.0106 & 0.3675$\pm$0.0104 & 0.3705$\pm$0.0098 & \textbf{0.3913}$\pm$0.0126 & 0.3777$\pm$0.0101 \\
    JAFFE & 0.8708$\pm$0.0317 & 0.9492$\pm$0.0088 & 0.7373$\pm$0.0807 & 0.9462$\pm$0.0126 & 0.9493$\pm$0.0234 & 0.9500$\pm$0.0134 & 0.9387$\pm$0.0260 & \textbf{0.9527}$\pm$0.0182 \\
    ORL & 0.4793$\pm$0.0205 & 0.6617$\pm$0.0331 & 0.5227$\pm$0.0303 & 0.6565$\pm$0.0362 & 0.6725$\pm$0.0250 & 0.6644$\pm$0.0299 & 0.6623$\pm$0.0299 & \textbf{0.6756}$\pm$0.0267 \\
    Coil20 & 0.6433$\pm$0.0223 & 0.6493$\pm$0.0249 & 0.5992$\pm$0.0433 & 0.6469$\pm$0.0259 & 0.6680$\pm$0.0230 & 0.6531$\pm$0.0234 & 0.6653$\pm$0.0208 & \textbf{0.6766}$\pm$0.0259 \\ \hline
    Average accuracy & 0.5400$\pm$0.0197 & 0.6530$\pm$0.0259 & 0.5443$\pm$0.0421 & 0.6585$\pm$0.0169 & 0.6865$\pm$0.0179 & 0.6761$\pm$0.0144 & 0.6590$\pm$0.0143 & \textbf{0.6970}$\pm$0.0156 \\
    Average rank & 7.4000 & 4.4000 & 7.1000 & 5.5000 & 2.7000 & 3.7000 & 4.0000 & 1.1000 \\
    \hline
    \end{tabular}}
    \label{tab_ACC}
\end{center}
\end{table}

\begin{table}[t]
\begin{center}
\caption{
The results of all clustering methods on 10 benchmark datasets, evaluated based on Normalized Mutual Information (NMI), are presented in the format of "mean ± standard deviation".
}
    \resizebox{\textwidth}{!}{
    \renewcommand\arraystretch{1.2}
    \begin{tabular}{l c c c c c c c c} \hline
    Datasets & MKKM & AMKKM & SB-KKM & LMKKM & MKKM-MR & MKKM-RK & SimpleMKKM & Ours \\ \hline
    heart & 0.0002$\pm$0.0001 & 0.3240$\pm$0.0000 & 0.1968$\pm$0.0400 & 0.2769$\pm$0.0000 & 0.3522$\pm$0.0000 & 0.2049$\pm$0.0000 & 0.0517$\pm$0.0000 & \textbf{0.3804}$\pm$0.0000 \\
    mfeat & 0.5989$\pm$0.0094 & 0.8925$\pm$0.0255 & 0.7089$\pm$0.0330 & 0.8751$\pm$0.0130 & 0.8808$\pm$0.0260 & 0.8578$\pm$0.0011 & 0.8289$\pm$0.0015 & \textbf{0.9128}$\pm$0.0137 \\
    Handwritten$\_$numerals & 0.6494$\pm$0.0144 & 0.8441$\pm$0.0465 & 0.7114$\pm$0.0363 & 0.8968$\pm$0.0025 & 0.9012$\pm$0.0011 & 0.8817$\pm$0.0119 & 0.8878$\pm$0.0015 & \textbf{0.9130}$\pm$0.0003 \\
    Flower17 & 0.4455$\pm$0.0141 & 0.4985$\pm$0.0088 & 0.3793$\pm$0.0141 & 0.4877$\pm$0.0081 & 0.5615$\pm$0.0073 & \textbf{0.5683}$\pm$0.0060 & 0.5546$\pm$0.0072 & 0.5669$\pm$0.0078 \\
    ProteinFold & 0.3806$\pm$0.0064 & 0.4038$\pm$0.0108 &	0.3474$\pm$0.0158 & 0.3942$\pm$0.0096 & 0.4417$\pm$0.0091 & 0.4281$\pm$0.0105 & 0.4127$\pm$0.0096 & \textbf{0.4627}$\pm$0.0090 \\
    YALE & 0.5462$\pm$0.0195 & 0.5710$\pm$0.0215 & 0.5107$\pm$0.0382 & 0.5737$\pm$0.0191 & 0.5858$\pm$0.0180 & 0.5651$\pm$0.0230 &	0.5902$\pm$0.0179 & \textbf{0.5950}$\pm$0.0173 \\
    Caltech101 & 0.6142$\pm$0.0044 & 0.6184$\pm$0.0057 & 0.5318$\pm$0.0055 & 0.6085$\pm$0.0063 & 0.6168$\pm$0.0057 & 0.6184$\pm$0.0057 & \textbf{0.6311}$\pm$0.0063 & 0.6238$\pm$0.0054 \\
    JAFFE & 0.8597$\pm$0.0237 & 0.9356$\pm$0.0131 & 0.8056$\pm$0.0478 & 0.9324$\pm$0.0136 & 0.9404$\pm$0.0189 & 0.9387$\pm$0.0146 & 0.9260$\pm$0.0197 & \textbf{0.9445}$\pm$0.0138 \\
    ORL & 0.6795$\pm$0.0112 & 0.8115$\pm$0.0167 & 0.7246$\pm$0.0160 & 0.8078$\pm$0.0173 & 0.8135$\pm$0.0157 & 0.8113$\pm$0.0151 & 0.8088$\pm$0.0176 & \textbf{0.8150}$\pm$0.0138 \\
    Coil20 & 0.7590$\pm$0.0148 & 0.7624$\pm$0.0142 & 0.7469$\pm$0.0185 & 0.7592$\pm$0.0131 & 0.7695$\pm$0.0120 & 0.7660$\pm$0.0150 & 0.7727$\pm$0.0137 & \textbf{0.7758}$\pm$0.0114 \\ \hline
    Average accuracy & 0.5533$\pm$0.0118 & 0.6662$\pm$0.0163 & 0.5663$\pm$0.0265 & 0.6612$\pm$0.0103 & 0.6863$\pm$0.0114 & 0.6779$\pm$0.0103 & 0.6465$\pm$0.0095 & \textbf{0.6990}$\pm$0.0079 \\
    Average rank & 7.3000 & 4.1000 & 7.5000 & 5.3000 & 2.8000 & 3.8000 & 4.0000 & 1.2000 \\
    \hline
    \end{tabular}}
    \label{tab_NMI}
\end{center}
\end{table}
\begin{table}[t]
\begin{center}
\caption{The results of all clustering methods on 10 benchmark datasets are presented in terms of "mean ± standard deviation" for the Purity (PUR) evaluation metric. }
    \resizebox{\textwidth}{!}{
    \renewcommand\arraystretch{1.2}
    \begin{tabular}{l c c c c c c c c} \hline
    Datasets & MKKM & AMKKM & SB-KKM & LMKKM & MKKM-MR & MKKM-RK & SimpleMKKM & Ours \\ \hline
    heart & 0.5556$\pm$0.0000 & 0.8222$\pm$0.0000 & 0.7550$\pm$0.0396 & 0.8000$\pm$0.0000 & 0.8333$\pm$0.0000 & 0.7630$\pm$0.0000 & 0.6407$\pm$0.0000 & \textbf{0.8444}$\pm$0.0000 \\
    mfeat & 0.6444$\pm$0.0167 & 0.9435$\pm$0.0318 & 0.7374$\pm$0.0506 & 0.9320$\pm$0.0180 & 0.9338$\pm$0.0348 & 0.9269$\pm$0.0006 & 0.8940$\pm$0.0015 & \textbf{0.9589}$\pm$0.0169 \\
    Handwritten$\_$numerals & 0.6603$\pm$0.0188 & 0.8300$\pm$0.0668 & 0.7632$\pm$0.0577 & 0.9508$\pm$0.0017 & 0.9536$\pm$0.0006 & 0.9388$\pm$0.0167 & 0.9437$\pm$0.0011 & \textbf{0.9610}$\pm$0.0002 \\
    Flower17 & 0.4506$\pm$0.0164 & 0.5214$\pm$0.0116 & 0.3721$\pm$0.0198 & 0.4860$\pm$0.0121 & 0.5929$\pm$0.0115 & 0.5840$\pm$0.0073 & 0.5853$\pm$0.0115 & \textbf{0.5959}$\pm$0.0066 \\
    ProteinFold & 0.3372$\pm$0.0084 & 0.3747$\pm$0.0153 & 0.3520$\pm$0.0192 & 0.3624$\pm$0.0110 & 0.4220$\pm$0.0114 & 0.4089$\pm$0.0135 & 0.3882$\pm$0.0131 & \textbf{0.4434}$\pm$0.0128 \\
    YALE & 0.5341$\pm$0.0257 & 0.5549$\pm$0.0255 & 0.4845$\pm$0.0412 & 0.5530$\pm$0.0274 & 0.5669$\pm$0.0247 & 0.5645$\pm$0.0276 & 0.5648$\pm$0.0239 & \textbf{0.5739}$\pm$0.0220 \\
    Caltech101 & 0.3857$\pm$0.0078 & 0.3928$\pm$0.0080 & 0.2878$\pm$0.0092 & 0.3784$\pm$0.0107 & 0.3911$\pm$0.0091 & 0.3928$\pm$0.0080 & \textbf{0.4142}$\pm$0.0125 & 0.4011$\pm$0.0097 \\
    JAFFE & 0.8722$\pm$0.0284 & 0.9492$\pm$0.0088 & 0.7734$\pm$0.0638 & 0.9462$\pm$0.0126 & 0.9498$\pm$0.0212 & 0.9500$\pm$0.0134 & 0.9392$\pm$0.0240 & \textbf{0.9531}$\pm$0.0155 \\
    ORL & 0.5214$\pm$0.0181 & 0.6962$\pm$0.0284 & 0.5667$\pm$0.0267 & 0.6922$\pm$0.0278 & 0.7038$\pm$0.0218 & 0.6984$\pm$0.0256 & 0.6973$\pm$0.0257 & \textbf{0.7082}$\pm$0.0225 \\
    Coil20 & 0.6664$\pm$0.0212 & 0.6790$\pm$0.0209 & 0.6533$\pm$0.0335 & 0.6780$\pm$0.0216 & 0.6827$\pm$0.0183 & 0.6835$\pm$0.0197 & 0.6846$\pm$0.0169 & \textbf{0.6940}$\pm$0.0202 \\ \hline
    Average accuracy & 0.5628$\pm$0.0162 & 0.6764$\pm$0.0217 & 0.5745$\pm$0.0361 & 0.6779$\pm$0.0143 & 0.7030$\pm$0.0153 & 0.6947$\pm$0.0132 & 0.6752$\pm$0.0130 & \textbf{0.7134}$\pm$0.0122 \\
    Average rank & 7.4000 & 4.3000 & 7.4000 & 5.5000 & 2.8000 & 3.6000 & 3.9000 & 1.1000 \\
    \hline
    \end{tabular}}
    \label{tab_PUR}
\end{center}
\end{table}

\begin{table}[t]
\begin{center}
\caption{
The results of all clustering methods on 10 benchmark datasets are reported in the form of "mean ± standard deviation" for the Adjusted Rand Index (ARI) evaluation metric.
}
    \resizebox{\textwidth}{!}{
    \renewcommand\arraystretch{1.2}
    \begin{tabular}{l c c c c c c c c} \hline
    Datasets & MKKM & AMKKM & SB-KKM & LMKKM & MKKM-MR & MKKM-RK & SimpleMKKM & Ours \\ \hline
    heart & -0.0040$\pm$0.0000 & 0.4131$\pm$0.0000 & 0.2631$\pm$0.0532 & 0.3576$\pm$0.0000 & 0.4424$\pm$0.0000 & 0.2738$\pm$0.0000 & 0.0755$\pm$0.0000 & \textbf{0.4726}$\pm$0.0000 \\
    mfeat & 0.4664$\pm$0.0141 & 0.8864$\pm$0.0420 & 0.6072$\pm$0.0543 & 0.8606$\pm$0.0230 & 0.8688$\pm$0.0442 & 0.8457$\pm$0.0011 & 0.7924$\pm$0.0025 & \textbf{0.9134}$\pm$0.0231 \\
    Handwritten$\_$numerals & 0.5180$\pm$0.0206 & 0.7636$\pm$0.0838 & 0.6225$\pm$0.0580 & 0.8952$\pm$0.0034 & 0.9009$\pm$0.0013 & 0.8740$\pm$0.0198 & 0.8810$\pm$0.0021 & \textbf{0.9150}$\pm$0.0004 \\
    Flower17 & 0.2661$\pm$0.0151 & 0.3258$\pm$0.0118 & 0.2023$\pm$0.0140 & 0.2972$\pm$0.0091 & 0.3950$\pm$0.0086 & 0.4005$\pm$0.0107 & 0.3927$\pm$0.0110 & \textbf{0.4021}$\pm$0.0102 \\
    ProteinFold & 0.1216$\pm$0.0071 & 0.1467$\pm$0.0159 & 0.1224$\pm$0.0168 & 0.1435$\pm$0.0173 & 0.1765$\pm$0.0136 & 0.1603$\pm$0.0123 & 0.1703$\pm$0.0144 & \textbf{0.1871}$\pm$0.0128 \\
    YALE & 0.3090$\pm$0.0272 & 0.3379$\pm$0.0303 & 0.2543$\pm$0.0465 & 0.3413$\pm$0.0286 & 0.3540$\pm$0.0260 & 0.3339$\pm$0.0265 & 0.3618$\pm$0.0252 & \textbf{0.3685}$\pm$0.0243 \\
    Caltech101 & 0.2319$\pm$0.0074 & 0.2388$\pm$0.0095 & 0.1304$\pm$0.0068 & 0.2234$\pm$0.0090 & 0.2362$\pm$0.0093 & 0.2388$\pm$0.0095 & \textbf{0.2565}$\pm$0.0100 & 0.2454$\pm$0.0088 \\
    JAFFE & 0.7727$\pm$0.0371 & 0.8946$\pm$0.0186 & 0.6655$\pm$0.0767 & 0.8887$\pm$0.0221 & 0.8992$\pm$0.0323 & 0.8970$\pm$0.0241 & 0.8803$\pm$0.0346 & \textbf{0.9048}$\pm$0.0232 \\
    ORL & 0.3213$\pm$0.0187 & 0.5443$\pm$0.0370 & 0.3774$\pm$0.0307 & 0.5367$\pm$0.0387 & 0.5497$\pm$0.0324 & 0.5458$\pm$0.0322 & 0.5404$\pm$0.0373 & \textbf{0.5526}$\pm$0.0313 \\
    Coil20 & 0.5720$\pm$0.0226 & 0.5830$\pm$0.0272 & 0.5372$\pm$0.0386 & 0.5773$\pm$0.0249 & 0.5977$\pm$0.0228 & 0.5881$\pm$0.0274 & 0.5993$\pm$0.0235 & \textbf{0.6023}$\pm$0.0216 \\ \hline
    Average accuracy & 0.3575$\pm$0.0140 & 0.5134$\pm$0.0276 & 0.3782$\pm$0.0340 & 0.5122$\pm$0.0176 & 0.5420$\pm$0.0191 & 0.5338$\pm$0.0164 & 0.4950$\pm$0.0161 & \textbf{0.5564}$\pm$0.0156 \\
    Average rank & 7.4000 & 4.3000 & 7.4000 & 5.3000 & 2.8000 & 3.7000 & 4.0000 & 1.1000 \\
    \hline
    \end{tabular}}
    \label{tab_ARI}
\end{center}
\end{table}

\begin{enumerate}
\item Our proposed algorithm demonstrates a significant performance advantage over existing MKKM algorithms by effectively considering redundant information present among kernels. For instance, on the heart dataset, our algorithm exhibits enhancements of 31.24$\%$, 38.02$\%$, 28.88$\%$, and 46.86$\%$ for ACC, NMI, PUR, and ARI, respectively, in comparison to the conventional MKKM algorithm. Similarly, on the mfeat dataset, our algorithm surpasses MKKM by 32.98$\%$, 31.39$\%$, 31.45$\%$, and 44.70$\%$, correspondingly, across these clustering evaluation metrics.

\item The AMKKM method stands as a firmly established benchmark model with acknowledged competence in addressing clustering problems through the assignment of average weights to kernel matrices. In comparison to the majority of other clustering methods, it consistently demonstrates comparable or even superior clustering performance. Nonetheless, our proposed model outperforms this benchmark, exemplifying superior performance in comparison.

\item Comparative experiments provide evidence of our algorithm's superior performance over SB-KKM, a difference that can be partly attributed to the incorporation of the composite kernel matrix. Through the amalgamation of multiple kernel matrices, we adeptly capture the inherent features of the data while diminishing redundant information. In contrast, the SB-KKM algorithm relies solely on a single kernel function to encapsulate data characteristics, a limitation that might affect its performance on specific datasets.

\item Our proposed method achieves the highest clustering results in 9 out of the 10 datasets and secures the second-best result in one dataset. In comparison to the MKKM-MR algorithm, which utilizes kernel correlation, and the MKKM-RK algorithm, which employs kernel dissimilarity, our method consistently demonstrates superior performance across all 10 datasets. This outcome validates the effectiveness of combining kernel dissimilarity and correlation in our approach. For instance, on the ProteinFold dataset, our algorithm outperforms MKKM-MR and MKKM-RK by $2.03\%$ and $2.14\%$ in terms of ACC, $2.30\%$ and $3.46\%$ in terms of NMI, $2.14\%$ and $3.45\%$ in terms of PUR, and $1.06\%$ and $2.68\%$ in terms of ARI, respectively. These substantial enhancements in performance provide compelling evidence of the advantages offered by our approach.

\item Our proposed method consistently exhibits a higher average accuracy compared to all other comparison methods, surpassing the second-best method by over $1\%$. Additionally, our method attains the highest average rank across all four clustering evaluation metrics.
\end{enumerate}

In our model, we use the Manhattan distance and the Frobenius inner product to measure the similarity between kernel matrices. A higher similarity between two kernels suggests more shared information. The Manhattan distance computes the absolute similarity of two kernel matrices, i.e., the sum of the absolute value differences $\sum_{i=1}^n\sum_{j=1}^n\left|\mathbf{A}(i,j)-\mathbf{B}(i,j)\right|$. This method effectively captures specific variations in each dimension and is sensitive to outliers or anomalous data points. However, it emphasizes the differences in specific numerical values without considering their relative magnitudes. On the other hand, the Frobenius inner product calculates the sum of the products of corresponding elements in two matrices $\sum_{i=1}^n\sum_{j=1}^n\mathbf{A}(i,j)\cdot \mathbf{B}(i,j)$. This method focuses on the overall dot product, emphasizing the relative relationships, proportions, and trends between elements, without considering specific numerical values. Larger values significantly impact the sum of the Frobenius inner product. Therefore, the expressions of similarity through the Manhattan distance and the Frobenius inner product provide distinct information and are not interchangeable.

In an ideal scenario, the high similarity between two kernels should correspond to a high kernel correlation, as represented by the Frobenius inner product, and a low kernel dissimilarity, as represented by the Manhattan distance. However, this theoretical consistency does not always hold true in practical applications. For instance, in the ProteinFold dataset (Table \ref{cor_dis_compare}), the observed kernel similarity values lack consistency. There are cases where despite having the highest correlation, the dissimilarity is also at its peak, as observed between $\mathbf{K}_1$ and $\mathbf{K}_{10}$.

Diverse datasets in practical scenarios exhibit distinct structures and characteristics, including variations in distribution, shape, and inter-sample relationships. As a result, kernel similarity results differ across datasets. Moreover, kernels with different types or parameters have unique properties, leading to varied methods of capturing relationships between data points. The inherent mapping characteristics of different kernel functions can disrupt the consistency of kernel similarity evaluation methods. For example, while linear kernel functions primarily capture linear relationships, Gaussian kernel functions demonstrate superior nonlinear modeling capabilities. These factors collectively influence the data representation, complicating the establishment of consistent trends when measuring the approximation level of kernel matrices through different methods. Therefore, relying solely on one approach is insufficient for effectively characterizing the approximation level between kernels.

However, simply reducing the number of kernel functions is not an ideal solution to address the inconsistency in kernel functions, as it may lead to a loss of the data’s non-linear features. Pursuing consistency by decreasing the number of kernel functions can degrade the algorithm’s performance, potentially devolving it into the kernel k-means algorithm. Therefore, our approach considers both measurement methods simultaneously, employing a consistency framework to constrain them. This provides a comprehensive depiction of the data’s non-linear features and the internal connections between different kernels. Experimental results confirm the benefits of considering both the relative sizes and absolute differences of elements when evaluating the overall similarity between kernels. This ensures the robustness and significant effectiveness of our method across various datasets.

\begin{table}[t]
\caption{Friedman test with the critical value $F(7, 63) = 2.159$. }
\begin{center}
    \renewcommand\arraystretch{1.2}
    \begin{tabular}{c c c c c }
    \hline
    {\footnotesize Friedman Test}  & {\footnotesize ACC} & {\footnotesize NMI} & {\footnotesize PUR} & {\footnotesize ARI} \\ \hline
    {\footnotesize $\tau_F$} & {\footnotesize 28.2781} & {\footnotesize 29.4146} & {\footnotesize 35.3662} & {\footnotesize 32.2664}\\ \hline
    \end{tabular}
    \label{table_friedman}
\end{center}
\end{table}

\subsubsection{Friedman Test and Post-hoc test}
The Friedman test is utilized to ascertain whether all algorithms are equivalent, assuming they share the same average rank \cite{2006Ttest}. We assess the performance of $8$ algorithms on 10 datasets, with a significance level set at $5\%$. The statistics are computed using the following equations:
\begin{equation}
    \tau_F=\frac{(n-1)\tau_{\chi^2}}{n(k-1)-\tau_{\chi^2}},
\end{equation}
where 
\begin{equation}
    \tau_{\chi^2}=\frac{12n}{k(k+1)}(\sum_{i=1}^nr_i^2-\frac{k(k+1)^2}{4}),
\end{equation}
and $r_i$ denotes the average rank. The variable $\tau_F$ follows an F distribution with degrees of freedom $k-1$ and $(k-1)(n-1)$. Observing Table \ref{table_friedman}, it is evident that the values resulting from the Friedman test for the four clustering evaluation metrics markedly exceed the critical value of $F(7.63)=2.159$. This discrepancy suggests the rejection of the null hypothesis of equivalent algorithms.

With the distinction between algorithms established, another test becomes crucial for assessing the distinctions among them. The Nemenyi test is employed to compute the critical difference (CD) of the average rank, defined as follows:
\begin{equation}
    CD=q_{\gamma}\sqrt{\frac{k(k+1)}{6n}},
\end{equation}
where $\gamma=0.05$ and the chosen critical value $q_{0.05}=3.031$ leads to $CD=3.3203$. When the average rank difference between two algorithms exceeds CD, they are considered significantly distinct. 

Utilizing our experimental results depicted in Figure~\ref{average_ranks}, it is evident that LMKKM, SB-KKM, and MKKM exhibit notable differences from our proposed algorithm. Both our approach and the MKKM-MR, MKKM-RK methods strive to mitigate redundant information among multiple kernels, resulting in more overlapping outcomes. Similarly, SimpleMKKM and A-MKKM include iterative kernel weight adjustments, resulting in a modest overlap with our algorithm. In summary, our method consistently surpasses all other algorithms used for comparison across all 4 metrics.

\begin{figure}[t]
\begin{center}
    \renewcommand\arraystretch{0.8}
    {\fontsize{10pt}{\baselineskip}\selectfont 
    \begin{tabular}{cc}
    \includegraphics[width=.45\linewidth]{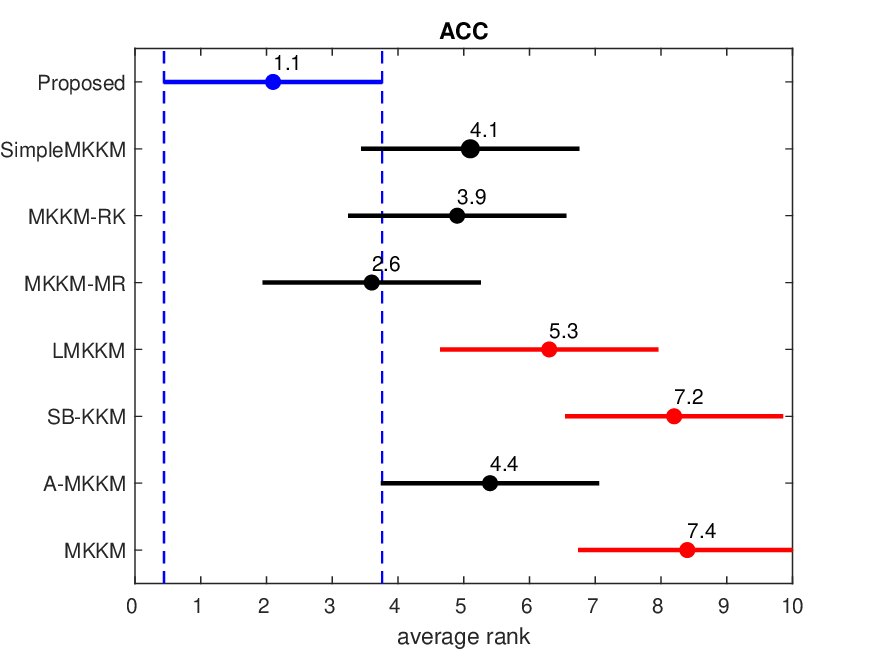} & \includegraphics[width=.45\linewidth]{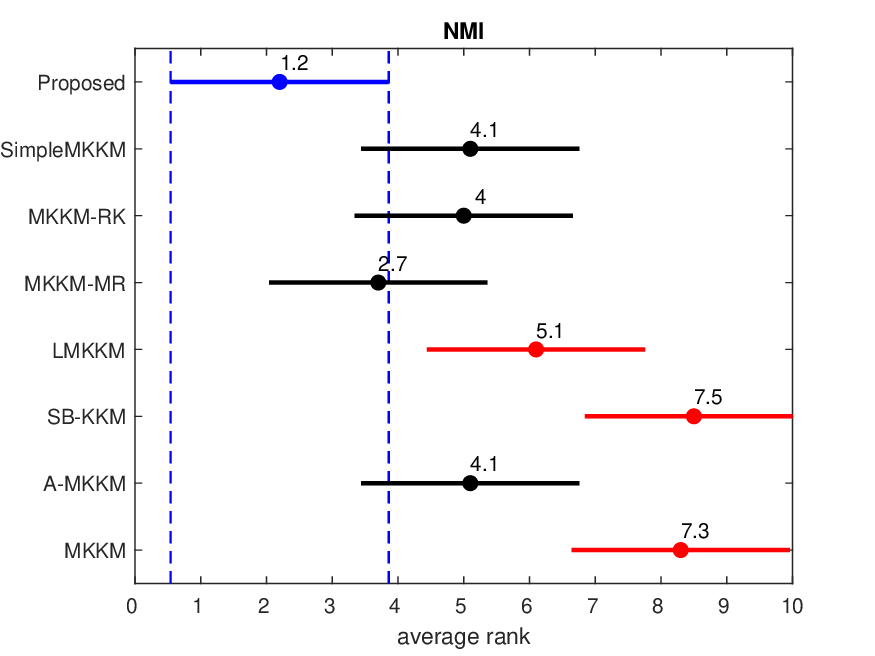} \\
    (a) & (b) \\
    \includegraphics[width=.45\linewidth]{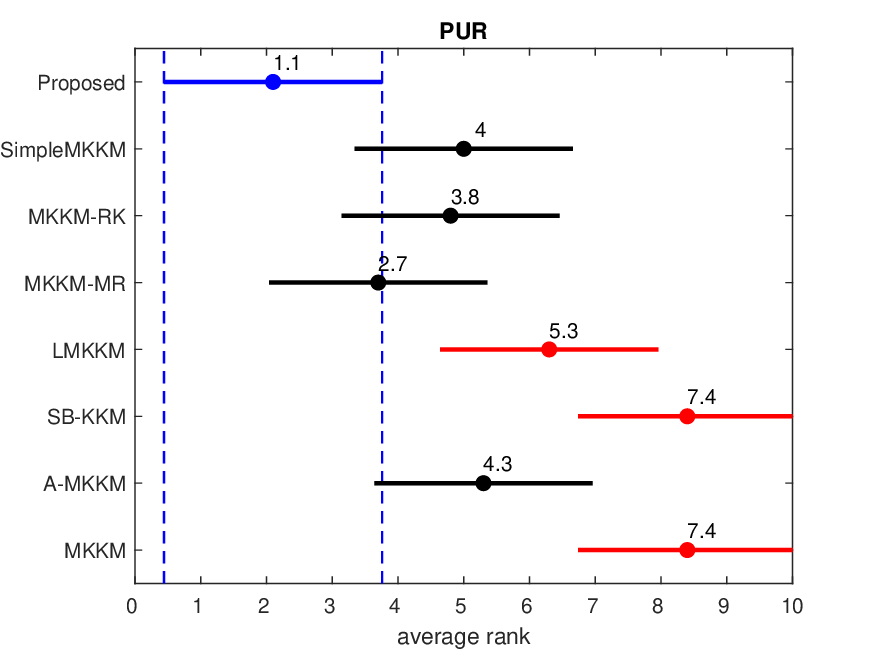} & \includegraphics[width=.45\linewidth]{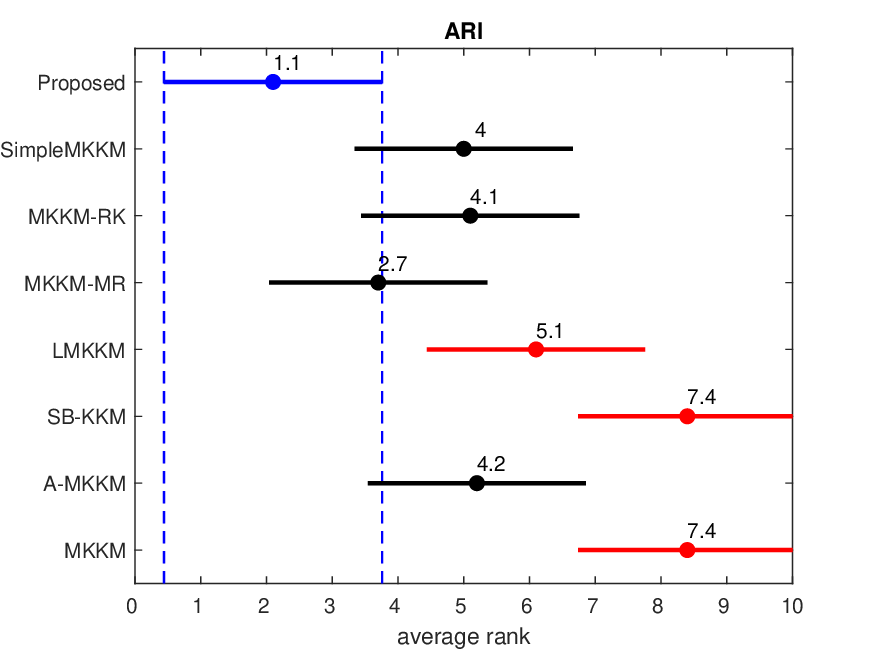} \\
    (c) & (d) 
    \end{tabular}
    }
    \caption{
The results of the Friedman test conducted on all comparison algorithms for the 4 clustering evaluation metrics: (a) ACC, (b) NMI, (c) PUR, and (d) ARI. The values depicted in the graph represent the average rank of the algorithms. The black and {\color{red}red} lines correspond to the algorithms that overlap and non-overlap with our algorithm, which is represented by the {\color{blue}blue} lines, respectively. 
 }
    \label{average_ranks}
\end{center}
\end{figure}

\subsubsection{Experiments on Large-scale Datasets}
In this section, we outlined the outcomes of comparative experiments conducted on three large-scale datasets detailed in Table \ref{tab_bigdata}. For both the MNIST10000 and USPS datasets, we built 12 base kernel matrices, excluding LMKKM due to its extended computational time on larger datasets. As data volume grows, the performance of these algorithms fluctuates across different datasets. For instance, while MKKM improves the ACC metrics on the USPS dataset, its performance significantly declines on the Flower102 dataset. Similar trends are evident for the AMKKM and SB-KKM algorithms. Conversely, simpleMKKM fares well on the Flower102 dataset but exhibits poorer performance on USPS and MNIST10000 datasets. Beyond the inherent differences among these datasets, the escalating noise accompanying larger datasets further impacts algorithmic robustness. Our experimental findings reveal that our proposed model’s clustering outcomes remain highly stable across these large datasets, demonstrating either optimal or near-optimal evaluation indices. This substantiates our study's effectiveness in mitigating data discrepancies by integrating kernel correlation and similarity, ensuring robustness across diverse datasets.

\begin{table}[t]
\begin{center}
\caption{
The results of 7 clustering methods on three large-scale datasets are reported in the form of "mean ± standard deviation" for the 4 clustering metrics.
}
\resizebox{\textwidth}{!}{
    \renewcommand\arraystretch{1.2}
    \begin{tabular}{lcccccccc}  \hline
    Datasets & Metric & MKKM & AMKKM & SB-KKM & MKKM-MR & MKKM-RK & SimpleMKKM & Ours\\  \hline
    \multirow{4}{*}{MNIST10000} & ACC & 0.5447$\pm$0.0010 & 0.5492$\pm$0.0003 & 0.5583$\pm$0.0387 & 0.5876$\pm$0.0138 & 0.5729$\pm$0.0010 & 0.3820$\pm$0.0103 & \textbf{0.6028}$\pm$0.0018\\
    &NMI&0.4882$\pm$0.0008&0.4936$\pm$0.0004&0.5188$\pm$0.0167&\textbf{0.5290}$\pm$0.0096&0.5028$\pm$0.0005&0.3404$\pm$0.0039&0.5203$\pm$0.0006\\
    &PUR&0.5974$\pm$0.0011 & 0.6022$\pm$0.0003 & 0.6046$\pm$0.0291 & 0.6363$\pm$0.0139 & 0.6209$\pm$0.0011 & 0.4181$\pm$0.0076 & \textbf{0.6491}$\pm$0.0013\\
    &ARI&0.3738$\pm$0.0010 & 0.3799$\pm$0.0006 & 0.4085$\pm$0.0267 & 0.4226$\pm$0.0173 & 0.3873$\pm$0.0011 & 0.2281$\pm$0.0065 & \textbf{0.4275}$\pm$0.0014\\ \hline
    \multirow{4}{*}{USPS} & ACC & \textbf{0.7013}$\pm$0.0001 & 0.6773$\pm$0.0007 & 0.6811$\pm$0.0437 & 0.7001$\pm$0.0003 & 0.3449$\pm$0.0001 & 0.2565$\pm$0.0001 & 0.6976$\pm$0.0093\\
    &NMI&0.6497$\pm$0.0002 & 0.6293$\pm$0.0008 & 0.6352$\pm$0.0173 & 0.6503$\pm$0.0004 & 0.3564$\pm$0.0003 & 0.2019$\pm$0.0001 & \textbf{0.6510}$\pm$0.0052\\
    &PUR&0.7687$\pm$0.0001 & 0.7516$\pm$0.0007 & 0.7460$\pm$0.0277 & 0.7693$\pm$0.0002 & 0.4531$\pm$0.0001 & 0.3470$\pm$0.0001 & \textbf{0.7703}$\pm$0.0091\\
    &ARI&\textbf{0.5880}$\pm$0.0002 & 0.5577$\pm$0.0008 & 0.5682$\pm$0.0331 & 0.5878$\pm$0.0003 & 0.2321$\pm$0.0003 & 0.0717$\pm$0.0002 & 0.5879$\pm$0.0080\\ \hline
    \multirow{4}{*}{Flower102} & ACC & 0.2252$\pm$0.0052 & 0.0425$\pm$0.0008 & 0.1196$\pm$0.0103 & 0.4022$\pm$0.0092 & 0.2249$\pm$0.0051 & 0.4209$\pm$0.0141 & \textbf{0.4232}$\pm$0.0071\\
    &NMI&0.4270$\pm$0.0029 & 0.1413$\pm$0.0015 & 0.2449$\pm$0.0151 & 0.5671$\pm$0.0049 & 0.4252$\pm$0.0024 & 0.5817$\pm$0.0059 & \textbf{0.5817}$\pm$0.0035\\
    &PUR&0.2798$\pm$0.0055 & 0.0556$\pm$0.0011 & 0.1320$\pm$0.0104 & 0.4634$\pm$0.0085 & 0.2785$\pm$0.0047 & 0.4812$\pm$0.0094 & \textbf{0.4833}$\pm$0.0061\\
    &ARI&0.1202$\pm$0.0042 & -0.0003$\pm$0.0002 & 0.0444$\pm$0.0057 & 0.2549$\pm$0.0062 & 0.1187$\pm$0.0033 & 0.2804$\pm$0.0106 & \textbf{0.2815}$\pm$0.0064\\ \hline
    \end{tabular}}
    \label{tab_bigdata}
    \end{center}
\end{table}

Our model's versatility in handling vast and diverse datasets has been demonstrated through rigorous experimental validation, particularly in applications like image recognition, pattern analysis, and data mining. Its consistent high performance showcases remarkable robustness, even when confronted with noise or intricate structures within the input data. This resilience is crucial in real-world scenarios, where data quality often fluctuates, especially in fields like medical imaging.

\subsubsection{Robustness of Clustering with Respect to Sample Size}

We examined how clustering outcomes are affected by varying sample sizes, analyzing the trends in performance for each algorithm as the sample size increased. The results, depicting the four clustering metrics, are displayed in Figure ~\ref{CCV_flower_trends} for the flower17 and CCV datasets. In both cases, the sample size was systematically augmented by 170 samples for flower17 and 200 samples for CCV, while maintaining a consistent number of kernels and classes.

\begin{figure*}[tp]
    \centering
    \renewcommand\arraystretch{0.8}
    {\fontsize{10pt}{\baselineskip}\selectfont 
    \begin{tabular}{cc}
    \includegraphics[width=.34\linewidth, trim={0 0 30 0},clip]{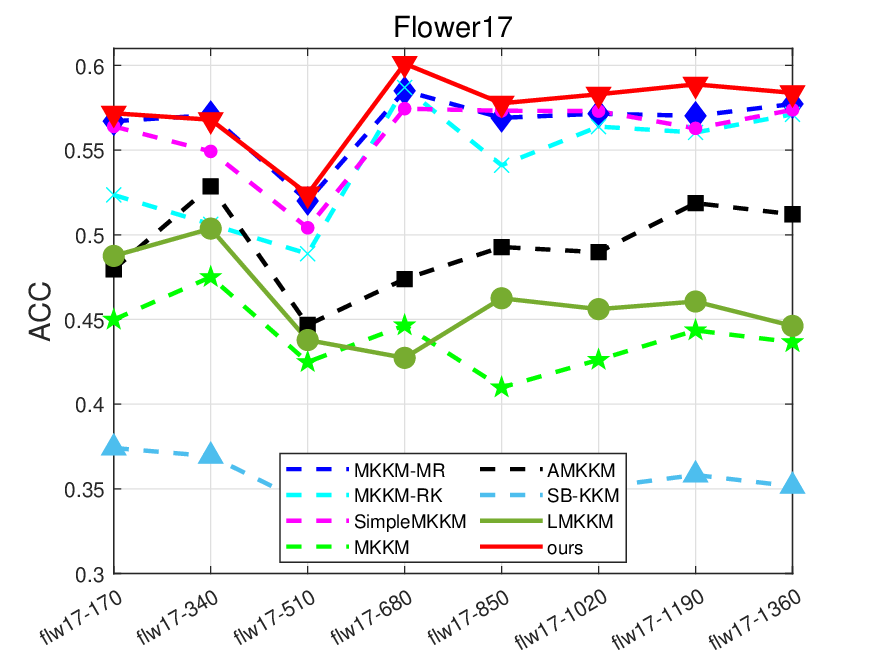} & 
    \includegraphics[width=.34\linewidth, trim={0 0 30 0},clip]{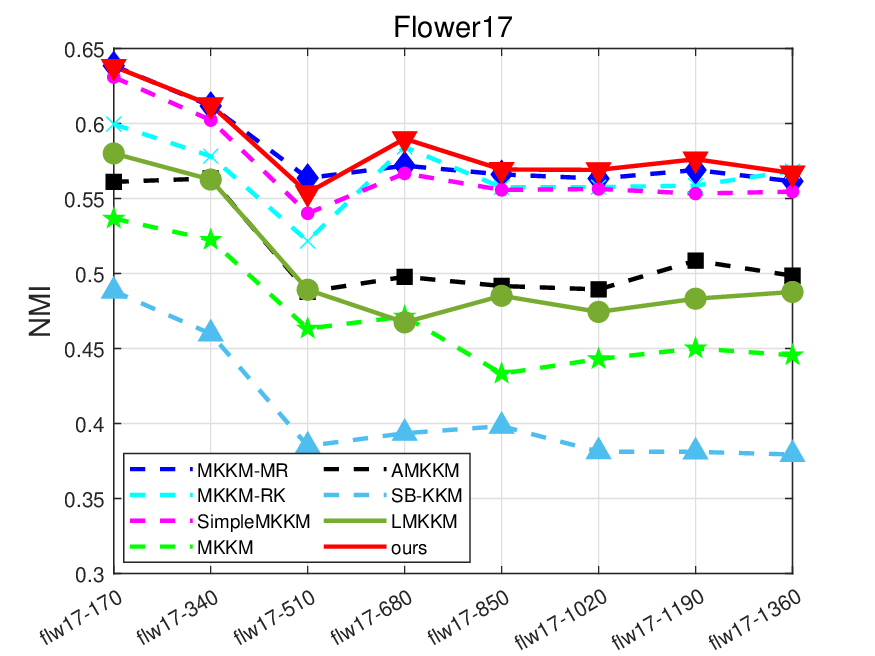} \\     
    (a) ACC on the flower17 & (b) NMI on the flower17 \\
    \includegraphics[width=.34\linewidth, trim={0 0 30 0},clip]{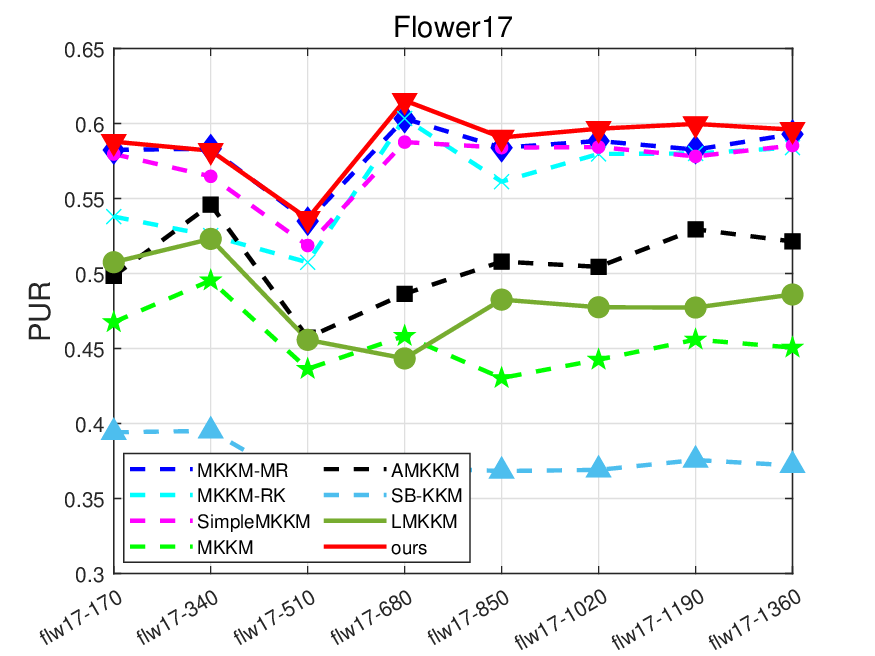} & 
    \includegraphics[width=.34\linewidth, trim={0 0 30 0},clip]{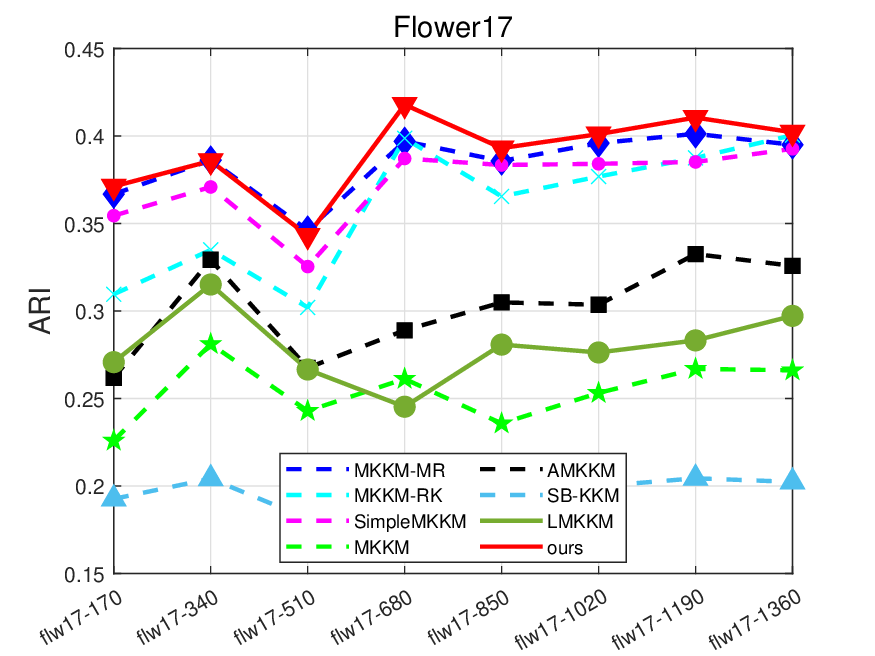}\\
    (c) PUR on the flower17 & (d) ARI on the flower17\\
    \includegraphics[width=.34\linewidth, trim={0 0 30 0},clip]{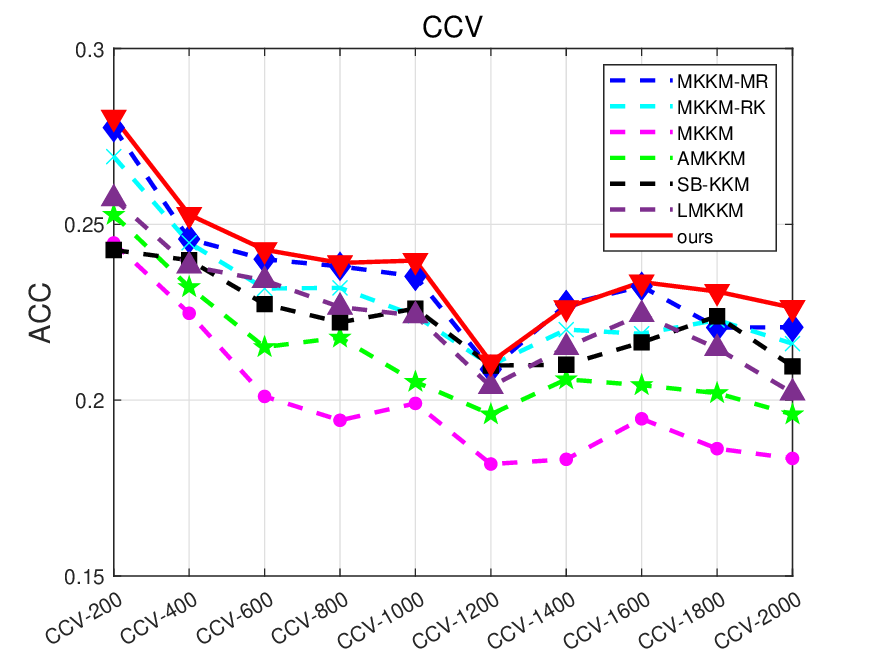} & \includegraphics[width=.34\linewidth, trim={0 0 30 0},clip]{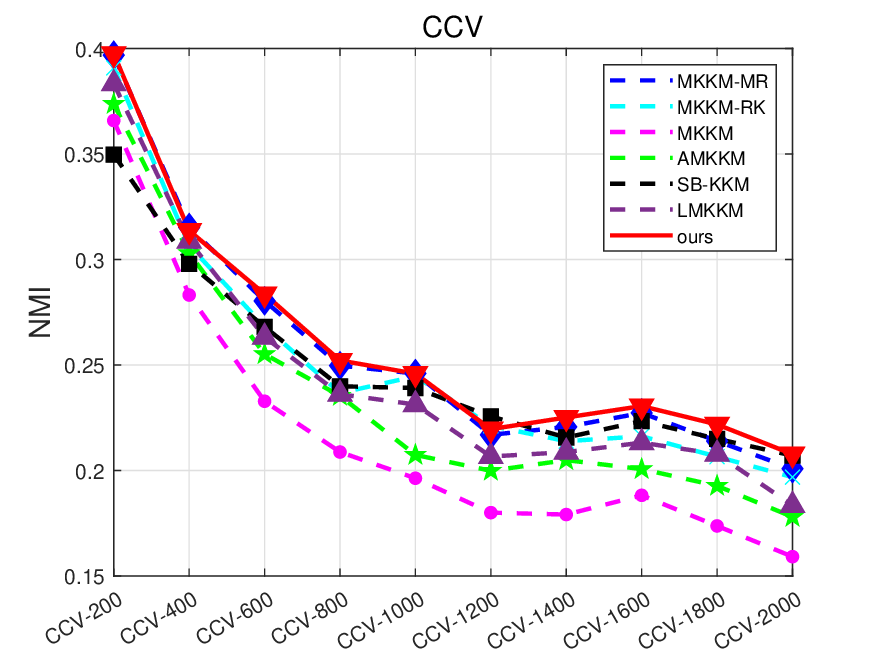} \\ 
    (e) ACC on the CCV & (f) NMI on the CCV \\
    \includegraphics[width=.34\linewidth, trim={0 0 30 0},clip]{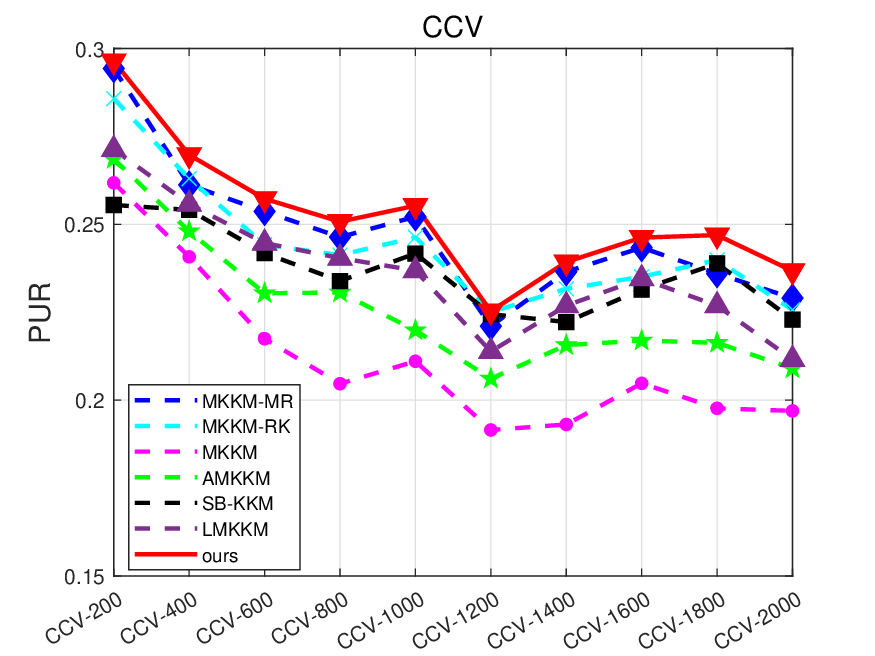} & \includegraphics[width=.34\linewidth, trim={0 0 30 0},clip]{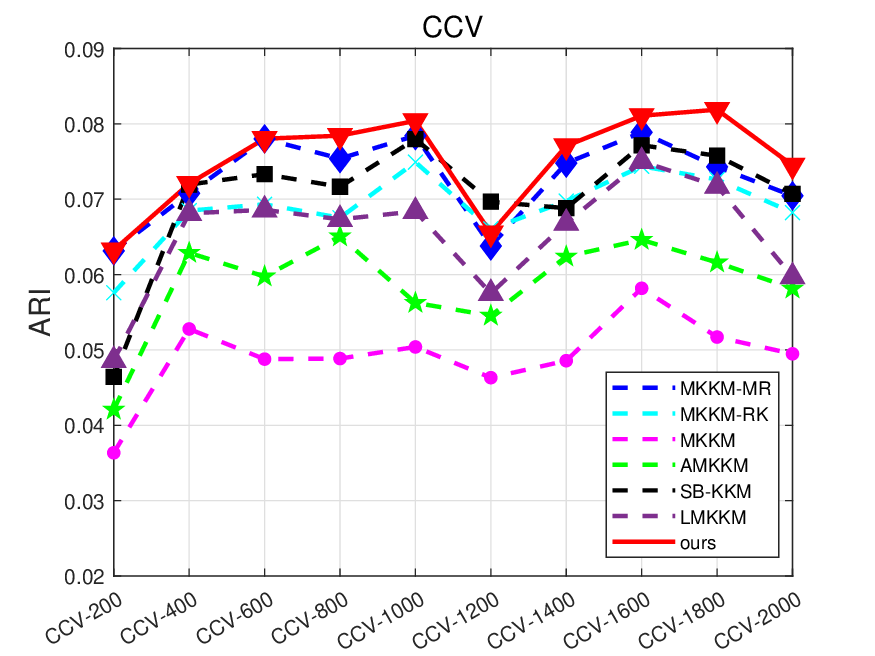}\\
    (g) PUR on the CCV & (h) ARI on the CCV \\
    \end{tabular}
    }
    \caption{
The trends in the 4 clustering metrics for all algorithms with an increasing number of sample elements are examined on the flower17 and CCV datasets. 
    }
    \label{CCV_flower_trends}
\end{figure*}

As observed from the figure, our proposed method consistently outperforms or demonstrates comparable performance to the MKM-MR and LMKKM methods for the majority of sample sizes. Exceptions arise when the sample size reaches 510 for the flower17 dataset and 1200 for the CCV dataset, where our method exhibits a minor underperformance across these four metrics in comparison to the mentioned methods. This observation underscores the importance of accounting for kernel correlation and dissimilarity to attain optimal clustering results.

Moreover, our approach consistently surpasses the single kernel method in performance. This discovery signifies that multiple kernel clustering methods provide a more comprehensive information framework compared to single kernel approaches. Effective fusion of multiple kernels can consequently lead to enhanced clustering outcomes.

\begin{figure}[t]
    \centering
    \renewcommand\arraystretch{0.8}
    {\fontsize{10pt}{\baselineskip}\selectfont 
    \begin{tabular}{cc}
    \includegraphics[height = 0.2\textheight, trim={30 0 50 0},clip]{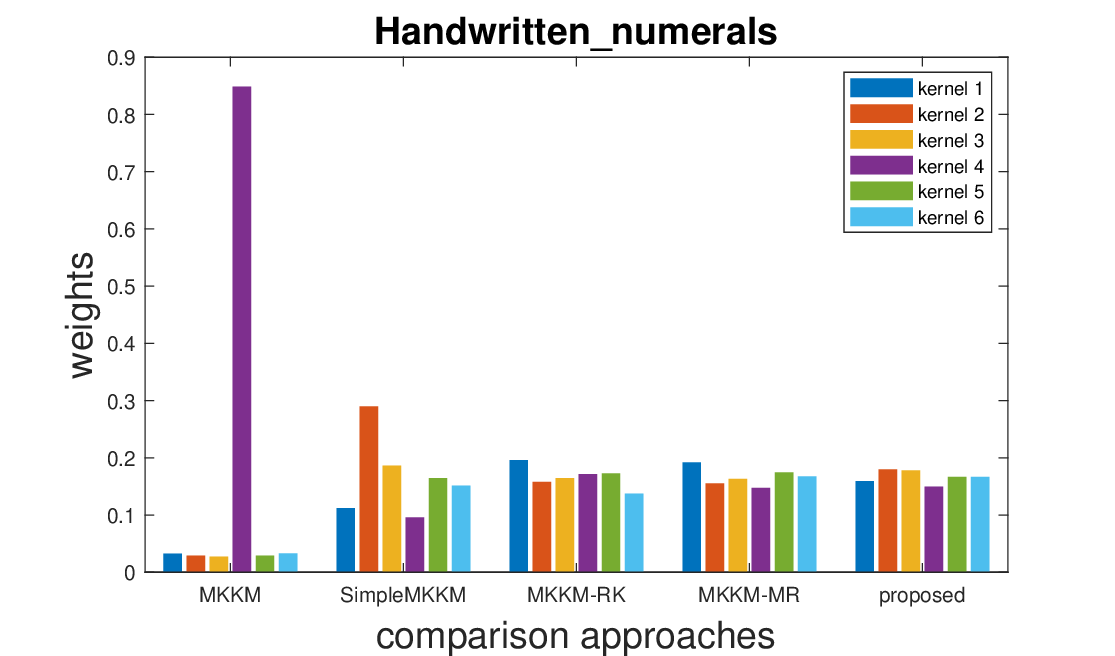} &
    \includegraphics[height = 0.2\textheight, trim={45 0 50 0},clip]{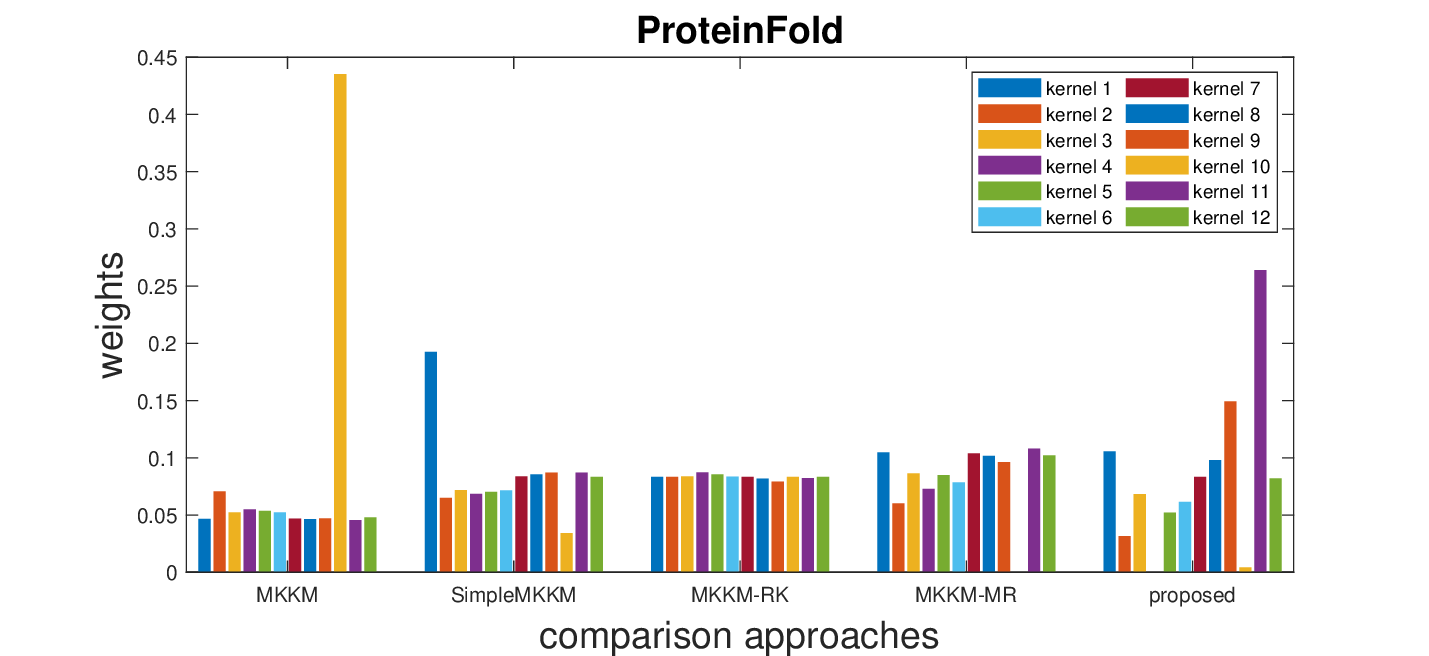} \\
    (a) Handwritten\_numerals & (b) ProteinFold
    \end{tabular}
    }
    \caption{
    A comparison of the learned kernel weights among different algorithms is conducted on 2 datasets. 
    }
    \label{kernelweight}
\end{figure}
\begin{figure}[!t]
    \centering
    \renewcommand\arraystretch{0.8}
    {\fontsize{10pt}{\baselineskip}\selectfont 
    \begin{tabular}{ccc}
    \includegraphics[width=.3\linewidth, trim={0 0 30 0},clip]{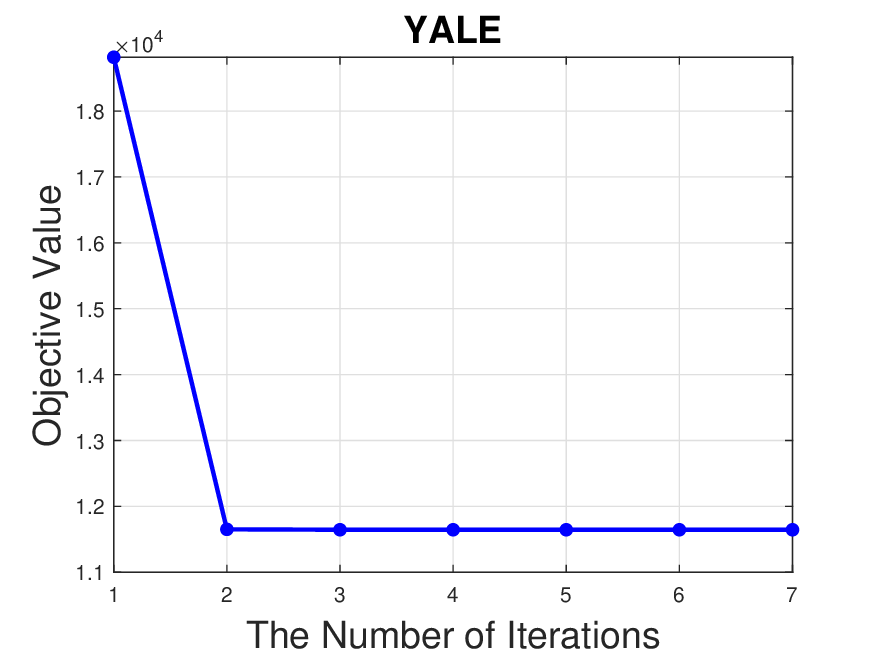} & \includegraphics[width=.3\linewidth, trim={0 0 30 0},clip]{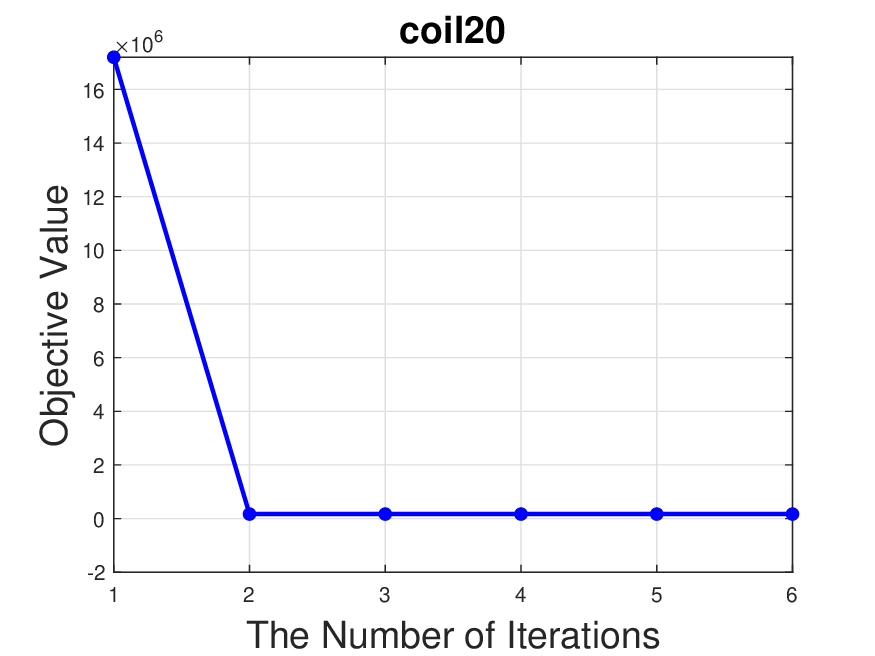} & \includegraphics[width=.3\linewidth, trim={0 0 30 0},clip]{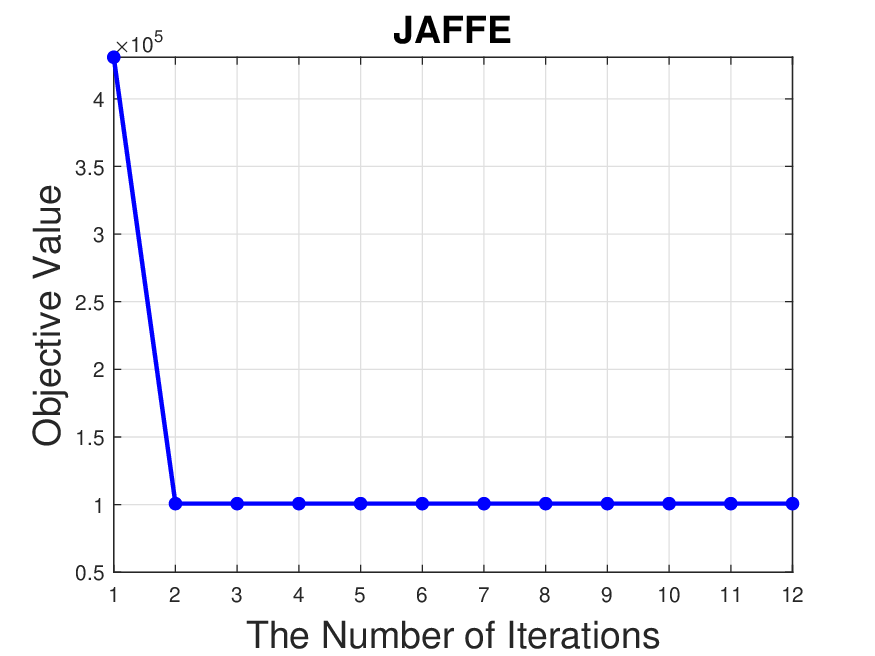} \\ 
    (a) YALE & (b) Coil20 & (c) JAFFE \\
    \includegraphics[width=.3\linewidth, trim={0 0 30 0},clip]{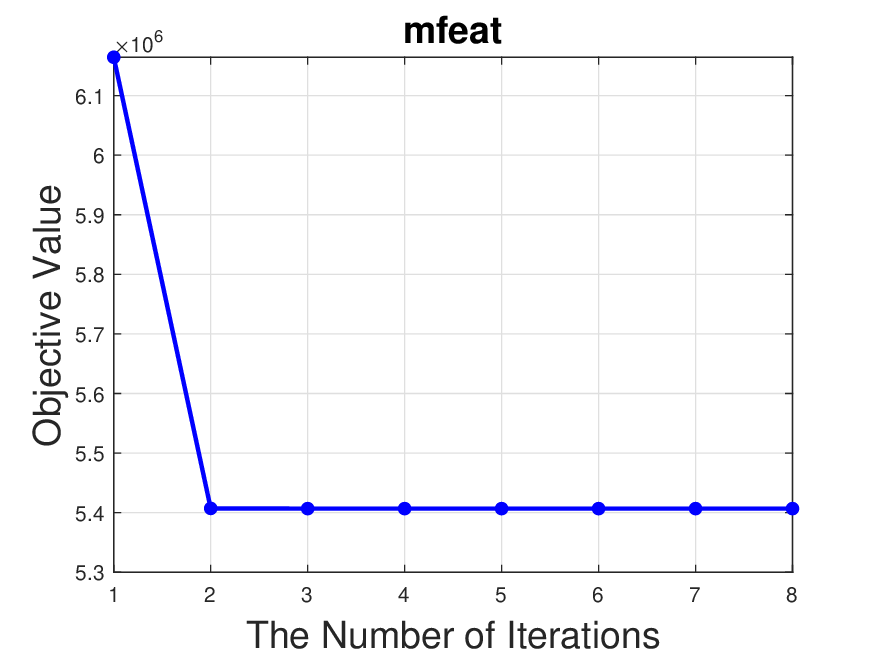} & \includegraphics[width=.3\linewidth, trim={0 0 30 0},clip]{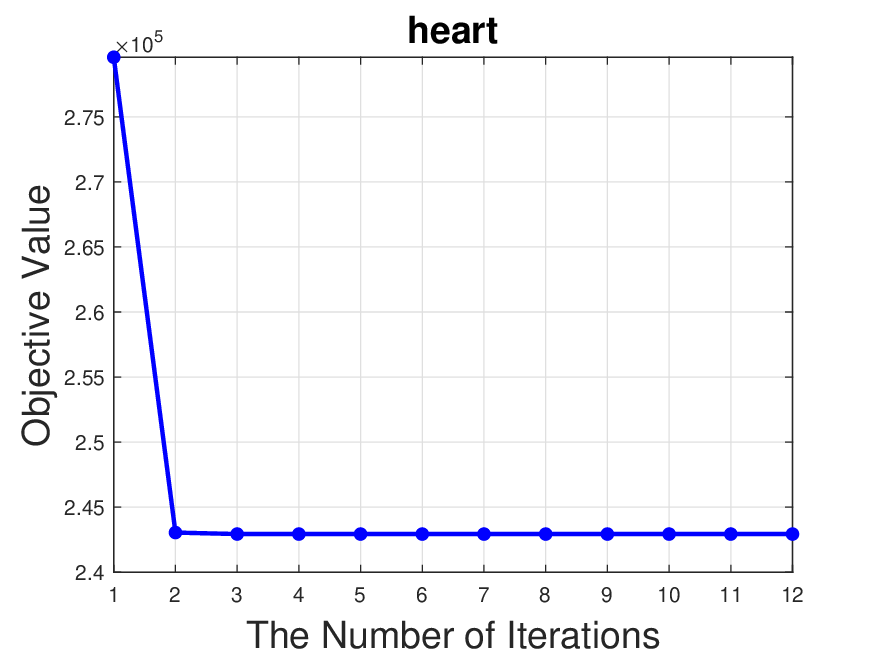} & \includegraphics[width=.3\linewidth, trim={0 0 30 0},clip]{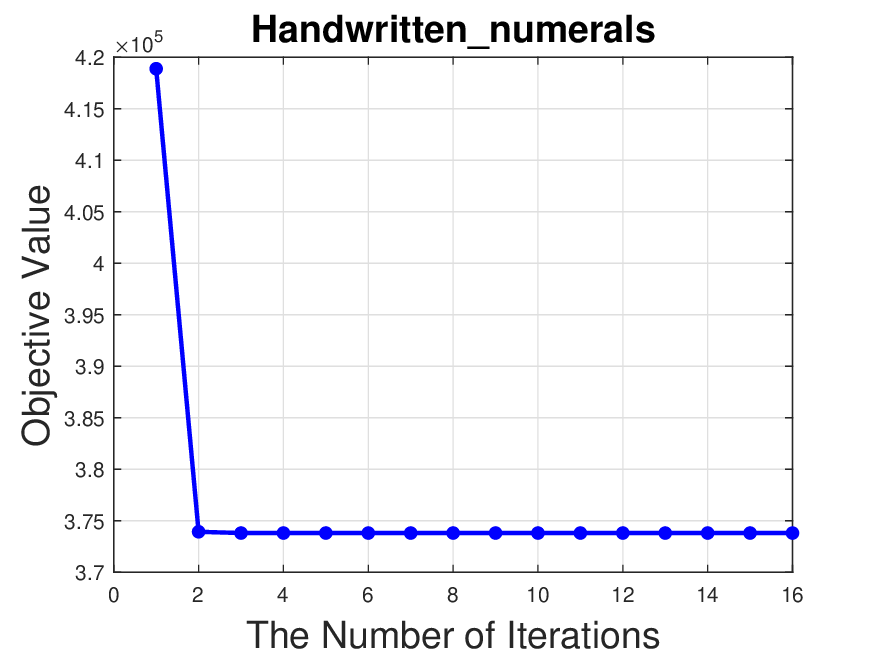} \\
    (d) mfeat & (e) heart & (f) handwritten\_numerals \\
    \end{tabular}
    }
    \caption{
        The iterative trend of the objective function for the proposed model is examined on 6 datasets. 
    }
    \label{objective_iter}
\end{figure}

\subsubsection{Kernel Weights Analysis}

Upon comparing the distribution of kernel weights across all algorithms, a substantial distinction was evident in the weight distribution of our proposed method, particularly in the case of the ProteinFold dataset, as depicted in Figure~\ref{kernelweight}. This highlights the significance of applying weight penalization to efficiently harness the correlation and dissimilarity inherent in the kernel matrix. Moreover, this observation illustrates that the kernel matrix can be harnessed to further amplify the quality of clustering outcomes. This strategy is pivotal since the optimal weight distribution remains unknown beforehand, and weight labels cannot be deduced from optimal clustering outcomes. Consequently, the imposition of penalties on redundant kernel matrices emerges as the singular viable strategy to attain heightened precision in clustering outcomes and extract the utmost insights from the data.

\subsubsection{Convergence analysis}

We examined the convergence behavior of our algorithm's objective function values. Consistent experiments on all datasets demonstrate a monotonic decrease in the objective function value with each iteration until convergence is achieved. The trend, as depicted in Figure~\ref{objective_iter}, illustrates rapid initial decreases in the objective function values, followed by a gradual approach towards near-optimal levels within approximately 2 iterations. The substantial acceleration in convergence leads to a significant reduction in the required computation time, resulting in a highly efficient model. In summary, our model exhibits remarkable convergence properties.

\section{Conclusions}
\label{section5}
This paper explores the interplay between kernel correlation and kernel dissimilarity, examining their relationships from multiple perspectives. We conclude that kernel correlation and dissimilarity extract information from multiple kernels in distinct ways and are not interchangeable. Hence, we introduce a model that seamlessly integrates both kernel correlation and dissimilarity to effectively penalize redundant information across multiple cores, thereby significantly improving clustering accuracy.

To streamline the solution process, we decompose the model into two convex optimization subproblems, enabling the application of an alternating minimization method for resolution. We evaluated our approach using 13 challenging datasets, and the experimental results firmly establish the superiority of our algorithm over state-of-the-art MKKM methods, confirming its effectiveness and excellence.

Our approach addresses the limitations of MKKM in handling kernel redundancy, providing a comprehensive representation of kernel relationships, and consequently achieving superior clustering outcomes that enhance overall clustering performance. However, current multiple kernel clustering methods still face challenges when dealing with large-scale datasets, especially those tailored for improving multiple kernel k-means (MKKM) clustering. These methods encounter hurdles when processing large-scale datasets due to the substantial size and computational intricacies associated with the original kernel matrix.

In future endeavors, it is imperative to enhance the algorithm's computational efficiency to meet the demands of effectively processing large-scale data. Simultaneously, exploring alternative methods for quantifying kernel redundancy becomes paramount, as it holds the potential to augment the algorithm's performance and resilience. Future research directions include a non-parametric exploration of multiple kernel clustering \cite{SimpleMKKM,2022dicsretemkkm}, as well as an investigation of other image segmentation clustering techniques applicable to MKKM. While our method has demonstrated improved performance, further exploration is warranted in the pursuit of developing exceptional multiple kernel clustering algorithms.

\section*{Acknowledgment}
    This work was supported by National Natural Science Foundation of China (No. 12061052), Young Talents of Science and Technology in Universities of Inner Mongolia Autonomous Region (No. NJYT22090), Natural Science Foundation of Inner Mongolia Autonomous Region (No. 2020MS01002), Innovative Research Team in Universities of Inner Mongolia Autonomous Region (No. NMGIRT2207), Inner Mongolia University Independent Research Project (2022-ZZ004), Inner Mongolia University Postgraduate Research and Innovation Programmes (No. 11200-5223737) and the network information center of Inner Mongolia University.
    The authors are also grateful to the reviewers for their valuable comments and remarks.

\bibliographystyle{elsarticle-num-names}
\bibliography{ref}

\end{document}